\documentclass[11pt]{article}
\usepackage[utf8]{inputenc}
\usepackage[T1]{fontenc}
\usepackage{newtxtext}
\usepackage[letterpaper,left=1.05in,right=1.05in,top=0.95in,bottom=0.95in,
            headheight=14pt,headsep=18pt]{geometry}
\usepackage{microtype}
\usepackage{xcolor}
\definecolor{runninggray}{gray}{0.40}
\usepackage{fancyhdr}
\renewcommand{\headrulewidth}{0.3pt}
\renewcommand{\headrule}{\hbox to\headwidth{\color{runninggray}\leaders\hrule height \headrulewidth\hfill}}
\fancypagestyle{plain}{\fancyhf{}\fancyfoot[C]{\small\thepage}\renewcommand{\headrulewidth}{0pt}}
\makeatletter
\renewcommand{\maketitle}{%
  \thispagestyle{plain}%
  \begin{center}
    {\fontsize{20}{24}\selectfont\bfseries\@title\par}
    \ifx\@author\@empty\else\vskip 14pt{\normalsize\begin{tabular}[t]{c}\@author\end{tabular}\par}\fi
    \ifx\@date\@empty\else\vskip 6pt{\small\@date\par}\fi
  \end{center}
  \vspace{10pt}}
\makeatother
\renewenvironment{abstract}
  {\begin{list}{}{\setlength{\leftmargin}{18pt}\setlength{\rightmargin}{18pt}}%
   \item\relax\small\noindent\textbf{Abstract.}\enspace\ignorespaces}
  {\end{list}\vspace{3pt}}
\usepackage{titlesec}
\titleformat{\section}{\large\bfseries}{\thesection}{0.7em}{}
\titleformat{\subsection}{\normalsize\bfseries}{\thesubsection}{0.7em}{}
\titleformat{\subsubsection}{\normalsize\bfseries}{\thesubsubsection}{0.7em}{}
\titlespacing*{\section}{0pt}{18pt plus 3pt minus 2pt}{8pt}
\titlespacing*{\subsection}{0pt}{12pt plus 2pt minus 1pt}{5pt}
\titlespacing*{\subsubsection}{0pt}{10pt plus 2pt minus 1pt}{4pt}

\usepackage{amsmath,amssymb,amsthm}
\usepackage{newtxmath}
\newtheorem{theorem}{Theorem}
\newtheorem{lemma}[theorem]{Lemma}
\newtheorem{proposition}[theorem]{Proposition}
\newtheorem{corollary}[theorem]{Corollary}
\theoremstyle{definition}

\newtheorem{example}{Example}
\theoremstyle{remark}

\DeclareMathOperator{\Var}{Var}
\DeclareMathOperator{\spn}{sp}

\usepackage{graphicx}
\graphicspath{{figures/}}
\usepackage{booktabs,longtable,array,float}
\usepackage[font=small,labelfont=bf,labelsep=period]{caption}
\usepackage[algo2e,ruled]{algorithm2e}
\SetAlFnt{\small}
\SetAlCapFnt{\small}
\SetAlCapNameFnt{\small\bfseries}

\usepackage[authoryear,round]{natbib}

\usepackage{xurl}
\definecolor{linkblue}{RGB}{30,65,115}
\usepackage[colorlinks=true,linkcolor=linkblue,citecolor=linkblue,
            urlcolor=linkblue]{hyperref}
\hypersetup{pdftitle={Average-Reward Policy Identification from a Single Continuing Trajectory},
            pdfauthor={Tongxi Wang}}

\title{Average-Reward Policy Identification\\from a Single Continuing Trajectory}
\author{Tongxi Wang\\[3pt]
  {\small Southeast University}\\[2pt]
  {\small\href{mailto:tongxi_wang@seu.edu.cn}{\texttt{tongxi\_wang@seu.edu.cn}}}}
\date{}

\begin{document}
\maketitle
\begin{abstract}
In reinforcement learning, an agent must explore an unknown environment to learn how to make good decisions. Unlike methods that can freely sample any state--action pair, an agent learning from a single trajectory must spend additional time navigating between states. It is therefore natural to expect that learning from a single trajectory requires more experience.
In this paper, we show that this additional cost is not always necessary. For average-reward reinforcement learning, we develop an algorithm that learns where to collect information while ensuring that the required samples can be obtained efficiently. We establish an upper bound on the number of interactions needed to identify a near-optimal policy. Furthermore, we prove a matching information theoretic lower bound, showing that, for a broad class of environments and at sufficiently high accuracy, learning from a single trajectory is no harder in the worst case than learning with unrestricted access to state action samples.
Our results reveal that the additional complexity cost of navigation can be avoided by adapting information gathering requirements to how the environment provides information.
\end{abstract}
\par\smallskip\noindent{\small\textbf{Keywords:}
reinforcement learning, policy identification, sample
complexity, adaptive sampling\par}

\section{Introduction}
Reinforcement learning uses experience to learn effective decisions in
an unknown environment. In tasks that operate continuously, performance
depends on the rewards accumulated over the long run. Average reward
captures this objective.
Policy identification asks how much experience is needed before the learner
can recommend a policy with a specified performance guarantee~\citep{tdk,zc}.
The amount of experience required depends on both statistical estimation
and data collection. Actions determine both observations and future sampling.

A generative model allows queries to any state--action pair, providing a
statistical benchmark that separates estimation from navigation. On one
continuing trajectory, obtaining the same useful samples may require many
intervening steps. Can a learner attain that benchmark while counting every
exploration and navigation step? We study this question with a fixed known
reward table. Even with immediate rewards known, the learner must estimate how
transitions determine future state visits and accumulated rewards.

\citet[Section~7]{tdk} pose this question for the $SAD/\epsilon^2$ rate.
The precision term of Online-DFE is $\widetilde O(SAD^2/\epsilon^2)$;
directly combining a generative planner~\citep{zc} and sample
collector~\citep{gosprl} also retains this extra diameter factor in its
diameter-based upper bound (Appendix~\ref{app:comparison}). Attaining the generative rate remains open.

We establish matching worst-case rates over a model class for continuing and
generative access under the conditions of Corollary~\ref{cor:access-matching}.
Both models have complexity $\widetilde\Theta(SA\overline D/\epsilon^2)$
for two-successor classes throughout $0<\epsilon\le1/32$, and
for unrestricted support at sufficiently high accuracy. The continuing
learner requires no structural parameters and counts every interaction,
including exploration and unsuccessful attempts. Its output is a
deterministic stationary policy, with stopping and all-state quality
guaranteed jointly. The lower-bound kernels share one known reward.
The refined upper bound also adapts to support and optimal bias span.
In these regimes, restricting where samples can be obtained does not
increase the worst-case complexity order: learning the sampling allocation
allows exploration and collection to fit within the statistical budget.

\paragraph{Our approach.}
Uniform requests can provide statistical coverage while requiring frequent
visits to states that the dynamics rarely reach. A generic collector
pays for this mismatch through a diameter factor. Two mathematical connections yield coverage and low collection cost.

First, inverse-count exploration produces an empirical allocation with
statistical coverage. A time-uniform trajectory flow bound shows that
the same visit counts can be covered by a balanced flow of comparable
mass, up to an additive cost. This holds for any nonanticipating
controller, so it applies directly to the adaptive exploration history.
The learned requests can then be collected at constant flow cost per
unit mass once startup is paid.

Second, nonuniform estimation converts these requests into an all-state
policy guarantee. A negative term from value truncation offsets the
error due to transient flow imbalance. The statistical error therefore
scales as $M^{-1/2}$ and the transient correction as $M^{-1}$, where $M$
is the requested mass. Combining this separation with low collection
cost gives the precision terms $SA(B_*+1)/\epsilon^2+SAD/\epsilon$.
An observable width certificate and a search over budgets complete the
learner, with all attempts counted and polynomial bit work on the same
success event.

\paragraph{Related work.}
Generative identification permits independent row sampling and
span-agnostic stopping~\citep{zc}; Online-DFE~\citep{tdk} pays for
continuing interaction. Behavior-policy methods such as
SAVIC+~\citep{lee} obtain prior-free all-state identification under
coverage parameters of the behavior policy, which are distinct from
controlled diameter. Discounted identification, offline learning and
regret impose different output or information conditions~\citep{lrp,fqi,focus26}.

Navigation-constrained allocation has direct precedents.
\citet{navigation} combine information requirements with stationary-flow
constraints and track a learned allocation to identify the unique optimal
policy in a discounted MDP. Their principal sample-complexity guarantee
is instance-dependent and asymptotic as $\delta\to0$.
Here the objective is fixed-accuracy average-reward identification, with
a nonasymptotic joint stopping-and-quality guarantee. Our finite-time
trajectory flow and controller bounds make allocation learning and
collection affordable without structural inputs, leading to the stated
classwise matching result. Optimistic collection~\citep{gosprl} already
makes prescribed requests executable; our analysis also pays for learning
which requests to make. The statistic
uses quantile planning and monotonicity~\citep{zzc}, span-agnostic
stopping~\citep{zc}, and source-state decoupling~\citep{zcplugin}.
Our analysis makes these tools effective with nonuniform weights fixed
after a pilot, and bounds the interactions needed to obtain those samples. Table~\ref{tab:comparison} and Appendix~\ref{app:comparison} compare the upper bounds.

Sections~\ref{sec:setting}--\ref{sec:theory} define the problem and state the results.
Sections~\ref{sec:algorithm}--\ref{sec:lower-overview} give the algorithm and the
upper- and lower-bound arguments; Section~\ref{sec:scaling} examines a prediction of the statistical analysis.

\section{Setting}\label{sec:setting}
There are $S\ge2$ states and $A\ge1$ actions. A \emph{row} $i=(s,a)$ is
the transition distribution $P_i$ of one state--action pair; write $K=SA$.
The unknown kernel is communicating, with controlled diameter
$D=\max_{s,t}\min_\pi\mathbb E_s^\pi\tau_t<\infty$.
The fixed reward table $r\in[0,1]^K$ and inputs
$\epsilon\in(0,1]$, $\delta\in(0,1)$ are rational and known.
The learner starts at any state and observes one trajectory, without reset
or generative access. It must return a deterministic stationary policy
whose average-reward gain is within $\epsilon$ of the optimal gain from
every initial state.

For a stationary policy $\pi$, define its gain by
\[
 g_P^\pi(s)=\lim_{T\to\infty}\frac1T
 \mathbb E_s^\pi\!\left[\sum_{t=0}^{T-1}r(s_t,a_t)\right].
\]
An optimal Bellman bias satisfies
$g_P^*+h^*(s)=\max_a\{r(s,a)+P_{sa}h^*\}$, with optimal gain
$g_P^*$ independent of the initial state. Writing $\tau$ for the
data-dependent stopping time, the required joint event is
\begin{equation}\label{eq:joint-guarantee}
 \Pr_P\!\left(\tau\le T_*(P,\epsilon,\delta),\quad
 \max_s\{g_P^*-g_P^{\widehat\pi}(s)\}\le\epsilon\right)\ge1-\delta.
\end{equation}
The learner stops using observed data. The bound $T_*$ is used only in analysis.

Let $d_i=|\operatorname{supp}P_i|$,
$J_{\rm supp}=\sum_i d_i$.
Define $B_*=\spn(h^*)$ for an optimal Bellman bias with $B_*\le D$, where
$\spn(h)=\max_s h(s)-\min_s h(s)$.
Write $A_{\rm mix}=D\sum_i\min\{d_i,D\}$ for the startup expression.

For the access-model comparison, a generative learner may adaptively
request independent successors from any row. For either access model,
the complexity of a model class is the smallest common deterministic budget $T$
for which the joint stopping-and-correctness event holds on every
instance in the class. This criterion controls both stopping and policy quality from every state.

\section{Main Results}\label{sec:theory}
\subsection{Upper Bound}
Theorem~\ref{thm:main-total} separates the cost of learning the allocation
from the subsequent precision cost. The support-dependent term is paid at
startup; the accuracy terms depend on the fixed reward through $B_*$.

\begin{theorem}[Policy identification from one trajectory]\label{thm:main-total}
Under the model of Section~\ref{sec:setting}, Algorithm~\ref{alg:main} returns a deterministic stationary
policy with all-state gain gap at most $\epsilon$, jointly with stopping
within
\[
 \widetilde O\!\left(D\sum_i\min\{d_i,D\}
       +\frac{SA(B_*+1)}{\epsilon^2}+\frac{SAD}{\epsilon}\right)
\]
interactions, with probability at least $1-\delta$.
It needs no support, diameter, mixing-time or minimum-transition input.
Its complete rational realization has polynomial bit work on that event.
The full controller, schedule and proof are in
Appendix~\ref{app:rowwise-mixed}.
\end{theorem}

The three terms pay for exploration and collection startup, statistical
precision, and the transient correction in this upper bound, respectively.
Since $B_*\le D$ and $\epsilon\le1$, Theorem~\ref{thm:main-total} implies
the diameter-dependent bound
\begin{equation}\label{eq:diameter-bound}
 \widetilde O(A_{\rm mix}+SAD/\epsilon^2).
\end{equation}
A weaker startup bound using total support follows from
$\sum_i\min\{d_i,D\}\le\min\{J_{\rm supp},SAD\}$; its expanded form is
recorded in Appendix~\ref{app:preliminaries}.

The guarantee counts pilots, navigation, surplus observations and every
unsuccessful attempt. Suppressed logarithms may depend on
$S,A,D,1/\epsilon,1/\delta$. Polynomial bit work holds on the same success
event; no expected stopping-time bound on failure histories is asserted.

Relative to $SAD/\epsilon^2$, the precision term has factor
$(B_*+1)/D+\epsilon$: small optimal span helps at high accuracy, while
startup and the $1/\epsilon$ transport term still depend on diameter.

\subsection{Matching Rates}\label{sec:matching}
The next consequence identifies when the continuing upper bound matches
hardness even with arbitrary row queries. Section~\ref{sec:setting} defines class complexity through a common
budget and joint guarantee.

\begin{corollary}[Matching continuing and generative access]\label{cor:access-matching}
Let $S\ge12$, $A\ge3$, $0<\epsilon\le1/32$, $0<\delta\le1/16$, and
\[
 \overline D\ge\max\{32,8(\lceil\log_{A-1}\lfloor S/3\rfloor\rceil+1)\}.
\]
Consider the class of communicating kernels with at most $S$ states,
at most $A$ actions per state, diameter at most $\overline D$, and total
support at most $J_0\in[2SA,S^2A]$, together with fixed known rational reward tables
in $[0,1]$. Require a deterministic stationary output policy that is
$\epsilon$-optimal from every state, jointly with stopping within a common
budget with probability at least $1-\delta$ on every instance.
If
\[
 \epsilon^2\min\{J_0/(SA),\overline D\}\le1,
\]
both continuing and generative access have class complexity
$\widetilde\Theta(SA\overline D/\epsilon^2)$.
The upper bound holds for every known reward table; the lower-bound
family shares one table across all kernels.
\end{corollary}

The matching condition absorbs startup: its ratio to
$SA\overline D/\epsilon^2$ is
$\epsilon^2\min\{J_0/(SA),\overline D\}\le1$.
Let $T^\star_{\rm gen}$
and $T^\star_{\rm cont}$ denote the complexities for the same class,
accuracy, confidence and joint stopping-and-output requirement in
Corollary~\ref{cor:access-matching}. For an absolute $c>0$, the comparison is
\[
 \underbrace{c\frac{SA\overline D}{\epsilon^2}\log\frac1\delta}
 _{\text{generative lower bound}}
 \;\le\;T^\star_{\rm gen}\;\le\;T^\star_{\rm cont}
 \;\le\;\underbrace{\widetilde O\!\left(\frac{SA\overline D}{\epsilon^2}\right)}
 _{\text{continuing upper bound}}.
\]
The middle inequality holds because generative access can simulate the
continuing learner. Theorem~\ref{thm:lower} gives the first inequality
with at most two successors per row; Theorem~\ref{thm:main-total} gives
the last. The resulting ranges, under the other conditions of
Corollary~\ref{cor:access-matching}, are:
\begin{center}
\begin{tabular}{@{}ll@{}}
\toprule Model class & Established matching range\\\midrule
At most two successors per row & $0<\epsilon\le1/32$\\
Unrestricted support & $\epsilon^2\min\{S,\overline D\}\le1$\\
\bottomrule
\end{tabular}
\end{center}
For unrestricted support outside this condition, access-model matching
remains open. The lower bound is classwise and uses a shared reward table;
Appendix~\ref{app:matching-summary} gives the complete comparison.

\subsection{Comparison with Prior Bounds}
Table~\ref{tab:comparison} separates startup from precision costs.
Substituting $B_*\le D$, our leading precision term is
$\widetilde O(SAD/\epsilon^2)$, compared with
$\widetilde O(SAD^2/\epsilon^2)$ in the other displayed bounds.
This improvement concerns the precision term of these upper bounds;
startup and span refinements are retained.
\begin{table}[H]\centering
\caption{Continuing policy-identification upper bounds, suppressing logarithms. $B_*=\spn(h^*)\le D$. The second row is a derived composition, derived in Appendix~\ref{app:comparison}.}\label{tab:comparison}
\small\setlength{\tabcolsep}{4pt}
\begin{tabular}{@{}>{\raggedright\arraybackslash}p{.24\linewidth}p{.28\linewidth}p{.41\linewidth}@{}}
\toprule Method & Startup & Precision\\\midrule
Online-DFE~\citep{tdk} & $S^2AD^3$ & $SAD^2/\epsilon^2$\\\addlinespace
Generative + GOSPRL\newline\citep{zc,gosprl} & $S^2AD^{3/2}$ & $SAD(B_*+1)/\epsilon^2$\\\addlinespace
Theorem~\ref{thm:main-total} & $D\sum_i\min\{d_i,D\}$ & $SA(B_*+1)/\epsilon^2$\newline $\quad{}+SAD/\epsilon$\\
\bottomrule\end{tabular}\end{table}

\section{Algorithm}\label{sec:algorithm}
Each attempt explores and freezes $w$, collects fresh prefixes of lengths
$q_i=\lceil Mw_i\rceil$, and computes $[L,U]$; width at most $\epsilon$
ends the search. Certificate validity holds for every positive $w$.
Pilot quality determines the required mass and collection cost.
Algorithm~\ref{alg:main} searches over pilot lengths and masses, clearing
memory between attempts while the physical trajectory continues.

\subsection{Sampling Allocation}\label{sec:allocation}
\begin{example}[Equal coverage, different collection cost]\label{ex:two-state}
With two identical actions per state, the chain below admits two
allocations with $G_P=4$ and flow costs $(1+p)/(2p)$ and $1$, for
$0<p\le1/2$.
\begin{center}
\includegraphics[width=\linewidth]{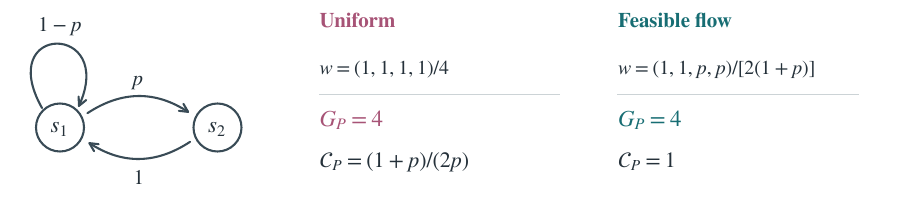}
\end{center}
Uniform requests assign half the samples to state 2, whose stationary
frequency is $p/(1+p)$. Matching this frequency lowers collection cost.
Weights are ordered as $(s_1,a_1),(s_1,a_2),(s_2,a_1),(s_2,a_2)$.
The identical actions isolate allocation geometry; equal $G_P$ alone
leaves policy-identification information unspecified
(Appendix~\ref{app:two-state-check}).
\end{example}

An allocation $w_i>0$, $\sum_iw_i=1$, divides statistical requests among
rows. Its statistical quality depends on inverse sample weights along
all feasible long-run occupations. Its collection cost also depends on
navigation: supplying a difficult row may require many visits elsewhere.
Let $\Omega(P)$ denote normalized stationary state--action flows. These
two requirements lead to
\begin{align}
 G_P(w)&=\max_{x\in\Omega(P)}\sum_i\frac{x_i}{w_i},\label{eq:coverage-quantity}\\
 \mathcal C_P(w)&=\min_{y\ge w}\left\{\sum_i y_i:
       \sum_a y_{sa}=\sum_i y_iP_i(s)\ \text{for every }s\right\}.\label{eq:collection-quantity}
\end{align}
The coverage coefficient $G_P(w)$ measures the first requirement.
The ideal collection cost $\mathcal C_P(w)$ is the minimum balanced mass
that supplies every requested proportion. For quotas $q_i=\lceil Mw_i\rceil$,
these definitions give
\[
 \max_{x\in\Omega(P)}\sum_i\frac{x_i}{q_i}\le\frac{G_P(w)}{M},
 \qquad \mathcal C_P(Mw)=M\mathcal C_P(w).
\]
The controller analysis in Section~\ref{sec:proof-overview} converts this
balanced-flow cost into a bound on finite random collection time,
including transient visits. Both quantities describe the unknown kernel
and guide the analysis; the learner estimates its allocation from visits.

A known kernel suggests how to meet both requirements. Choose
\[
 w^\star\in\arg\max_{w\in\Omega(P)}\sum_i\log w_i.
\]
Communication ensures a positive feasible flow, so a maximizer is positive.
First-order optimality gives
\[
 \sum_i\frac{x_i-w_i^\star}{w_i^\star}\le0
 \quad\text{for every }x\in\Omega(P),
 \qquad G_P(w^\star)=SA,\quad\mathcal C_P(w^\star)=1.
\]
The inequality gives coverage $SA$, with equality at $x=w^\star$;
feasibility gives unit collection cost. This known-model design
(Appendix~\ref{app:static-design}) motivates learning both properties
to within constant or logarithmic factors from finite interaction.

The next controller learns coverage of order $SA$ and bounded collection
cost from visits. Statistical estimation then determines the mass to request.

\subsection{Exploration and Collection}
The same controller serves two purposes: inverse-count rewards encourage
exploration during the pilot, and indicators of unfinished quotas guide
collection. If $c_i$ counts pilot visits, the reward $b_t(i)=1/(c_i+1)$
is the derivative of $\sum_i\log(c_i+1)$: another visit is more valuable
at a less visited row. After $n$ interactions the pilot freezes
$w_i=(c_i+1)/(n+K)$; write $w_{\min}=\min_iw_i$.
The smoothed weights are positive, and the analysis establishes both
coverage and low collection cost for these empirical proportions.

For each row, the controller forms two upper predictions of the expected
next-state value. The first maximizes over a confidence region for the
whole transition distribution and benefits from sparse support. The
second uses recent scalar evaluations at value vectors fixed before
each transition, giving an error bound independent of support size.
On their common confidence event, the smaller prediction remains
optimistic. Using it in each rounded Bellman backup lets different rows
benefit from different bounds, yielding the startup term
$D\sum_i\min\{d_i,D\}$. This controller builds on optimistic sample
collection~\citep{gosprl}; its analysis handles the decreasing auxiliary
rewards used in both phases. Appendix~\ref{app:mixed-spec} specifies these updates and restarts.

After the pilot, a fresh controller collects $q_i=\lceil Mw_i\rceil$
successors per row using reward $\mathbf1\{k_i<q_i\}$, where $k_i$ is
its visit count. The nominal mass $M$ differs from $N=\sum_iq_i$.
All visits update navigation information; the statistic uses the first
$q_i$ post-pilot successors per row. Within an invocation, restarts
retain visits and nested confidence regions but clear scalar windows.
Between attempts all algorithm memory is cleared and the physical
trajectory continues. Each attempt costs at most $2U_j$, with confidence
allowances specified in Appendix~\ref{app:outer-schedule}.

\begin{algorithm2e}[H]
\caption{Continuing policy identification}\label{alg:main}
\DontPrintSemicolon
\KwIn{States, actions, known rational $r$, accuracy $\epsilon$, confidence $\delta$; one continuing trajectory.}
\KwOut{Deterministic stationary policy $\widehat\pi$.}
\For{stage $j=0,1,\ldots$}{
 Set cap $U_j=2^j$ and each event allowance $\eta_j=\delta/[64(j+1)^4]$\;
 \For{$(p,m)\in\{0,\ldots,j\}^2$, lexicographically}{
  Clear algorithm memory at the current physical state; run the $2^p$-step pilot and freeze $w_i=(c_i+1)/(2^p+K)$\;
  Set $M=2^m$; \textbf{if} $M<2/w_{\min}$, \textbf{continue} to the next attempt\;
  Set $q_i=\lceil Mw_i\rceil$; start a fresh zero-count quota controller and collect these prefixes for at most $U_j$ interactions\;
  \textbf{if} quotas are incomplete, \textbf{continue} to the next attempt\;
  Compute residual-corrected $[L,U]$ and $\widehat\pi$ from these prefixes\;
  \textbf{if} $U-L\le\epsilon$, \textbf{return} $\widehat\pi$\;
 }
}
\end{algorithm2e}

\subsection{Stopping Rule}
Since sufficient pilot and sample sizes are unknown, stopping requires
both a lower bound on the output policy and an upper bound on optimal gain:
\[
 L\le g_P^{\widehat\pi}(s)\quad\text{for every }s,\qquad g_P^*\le U.
\]
The learner accepts when $U-L\le\epsilon$. On the common confidence
event, the bounds hold for any positive frozen allocation; its geometry
determines how quickly the width shrinks.
For each row, the quantile functional $F_i^-(v)$ truncates the upper tail
of $v$ under the fixed-prefix empirical distribution and subtracts an
empirical-variance penalty. Reflection gives $F_i^+(v)=-F_i^-(-v)$.
For each chosen discount, solve
\[
 V_\gamma^\pm(s)=\max_a\{r_{sa}+\gamma F_{sa}^\pm(V_\gamma^\pm)\}.
\]
The fixed-point concentration argument bounds the true expectation
$P_iv$ between these predictions at the value vectors used in this
analysis. The lower solve uses $\gamma_-=1-1/(M+1)$ and its greedy policy
$\widehat\pi$. At its exact fixed point, $F^-\le P$ implies
\[
\begin{aligned}
 r_{\widehat\pi}+P_{\widehat\pi}V^- -V^-
 &=P_{\widehat\pi}V^- -\gamma_-F^-_{\widehat\pi}(V^-)\\
 &\ge(1-\gamma_-)P_{\widehat\pi}V^-\\
 &\ge(1-\gamma_-)\min_sV^-(s)\,\mathbf1.
\end{aligned}
\]
Summing conditional expectations along a trajectory telescopes the
value terms. Their bounded endpoint difference vanishes after division
by time, giving the all-state lower bound
$L=(1-\gamma_-)\min_sV^-(s)$.
The upper prediction gives the reverse inequality for every action;
a finite dyadic discount search yields
$U=\min_\gamma(1-\gamma)\max_sV^+_\gamma(s)$.
Both endpoints are clipped to $[0,1]$. Algorithm~\ref{alg:main} tests the
width after each completed batch, with the numerical residual corrections
and finite fallback in Appendix~\ref{app:nonuniform-statistic}.

\section{Proof Overview}\label{sec:proof-overview}
Two connections establish the interaction bound. The first turns an
adaptive exploration history into requests with low balanced-flow cost.
The second converts their nonuniform sample counts into an all-state
certificate, with a linear transient correction. We state their
consequences first, then explain the arguments in
Sections~\ref{sec:trajectory-flow} and~\ref{sec:nonuniform-estimation}.

\subsection{Interaction Bound}
\begin{lemma}[Coverage and collection cost of the pilot allocation]
\label{lem:main-pilot}
For a fixed inverse-count invocation, on the joint controller and
trajectory flow-balance event specified in Appendix~\ref{app:rowwise-mixed},
a pilot of length
$n=\widetilde O(A_{\rm mix})$ suffices to give
\[
 G_P(w)\le32K(1+\ln(n+1)),\qquad
 \mathcal C_P(w)\le\tfrac32,\qquad
 w_{\min}^{-1}\le n+K=\widetilde O(A_{\rm mix}).
\]
The sufficient length is an analysis quantity unknown to the learner.
\end{lemma}

\begin{lemma}[Sample requirements and collection cost]\label{lem:main-demand}
Fix a finite pilot history and its positive frozen $w$. For each prescribed
batch $M\ge2/w_{\min}$, the virtual fixed-prefix statistic is sound with
conditional probability at least $1-\eta$. At a sufficiently large prescribed
mass it also accepts on that event, with sufficient mass of order
\[
 M_* =\widetilde O\!\left(
    \frac{G_P(w)(B_*+1)}{\epsilon^2}
    +\frac{DG_P(w)}{\epsilon}+\frac1{w_{\min}}\right).
\]
If $\mathcal C_P(w)\le3/2$, the combined controller collects the
$q_i=\lceil Mw_i\rceil$ successors in
$T_{\rm collect}(M)=\widetilde O(A_{\rm mix}+M)$ interactions on its
confidence event, when $M\ge2/w_{\min}$ and the cap is sufficient.
\end{lemma}
Here $\eta$ is a local conditional failure allowance. The statistical
statement is defined on virtual row streams even when a collection cap
fails; the collection statement requires a sufficient cap.

To see how the lemmas combine, a sufficient pilot gives
$G_P(w)=\widetilde O(K)$, $\mathcal C_P(w)\le3/2$, and
$1/w_{\min}=\widetilde O(A_{\rm mix})$. Substitute these into the
sufficient mass in Lemma~\ref{lem:main-demand}, then its collection bound:
\[
 n+T_{\rm collect}(M_*)=
 \widetilde O\!\left(A_{\rm mix}
       +\frac{K(B_*+1)}{\epsilon^2}+\frac{KD}{\epsilon}\right).
\]
The enumeration reaches such
an attempt with only logarithmic overhead, while a summable allocation of
conditional failure probabilities keeps the total below $\delta$
(Appendix~\ref{app:outer-schedule}). Certificates remain sound at poor
pilots. Conditioning on each finite entry history preserves the law of
the unused samples, including on paths where collection later fails.

The pilot controls statistical demand through $G_P(w)$, and the
collector converts that demand into interactions. The next two arguments
establish these properties for the same random allocation.

\subsection{Trajectory Flow}\label{sec:trajectory-flow}
The first connection covers realized visit counts by a balanced flow.
\begin{proposition}[Flow cost of an adaptive trajectory]\label{prop:main-flow}
Under the model of Section~\ref{sec:setting}, let $c_i(t)$ count visits
to row $i$ through time $t$. For any nonanticipating controller and
$\eta\in(0,1)$, with probability at least $1-\eta$, simultaneously for
all integers $t\ge0$,
\begin{equation}\label{eq:main-history-flow}
 \mathcal C_P(c(t))\le\tfrac54t+
 O\!\left(D[S\log(e(A+2))+\log(1/\eta)]\right).
\end{equation}
\end{proposition}
Thus every adaptive trajectory supplies counts that a balanced flow can
cover at comparable mass, after an additive transient cost. This bound
applies throughout exploration, including before stationary behavior
emerges. The controller converts this flow bound into a collection-time guarantee.
The additive form matters: after normalization by the pilot length, the
transient contribution decreases, while the leading flow cost stays
bounded. This allows a single sufficiently long pilot to support larger
subsequent requests. The proof below establishes this property directly
for the visit counts generated during learning, including their
history-dependent choice of actions.

\paragraph{Proof idea.}
Let $c_i(t)$ be row counts through time $t$ and $i_u=(s_u,a_u)$.
For a fixed potential $h$, replacing observed successors by their
conditional means leaves the exact residual
\[
 \sum_i c_i(t)\{h(s_i)-P_i h\}
 =h(s_0)-h(s_t)+M_t(h),\quad
 M_t(h)=\sum_{u<t}\{h(s_{u+1})-P_{i_u}h\}.
\]
The endpoint term is bounded by the span of $h$. The martingale term is
more delicate: the potential that prices a realized request is selected
after observing the trajectory. Applying concentration to that fitted
potential as if it were fixed would be invalid.

The dual characterization of the collection cost is
\[
 \mathcal C_P(c)=\max_{b\ge0,h}
 \left\{c\cdot b:b_i+P_i h-h(s_i)\le1\ \text{for all }i\right\}.
\]
An optimizing bias can be replaced by one in
$\mathcal F=\{h\in[0,D]^S:P_i h-h(s_i)\le1\}$.
This deterministic polytope has at most $\binom{K+2S}{S}$ vertices.
For each vertex, a variance identity and two all-time martingale bounds
control the residual. The same event then holds for every potential in
$\mathcal F$ by linearity of $M_t(h)$, allowing the path-dependent
optimizer to be selected afterwards. This proves the stated flow bound;
Proposition~\ref{rc:prop:vertex-history} gives the full argument and constants.
The deterministic vertex set is used only in the analysis.

This bound makes the link to the empirical allocation explicit.
Subadditivity and the regenerative cost bound
$\mathcal C_P(\mathbf1)\le K(D+1)$ imply
\[
 \mathcal C_P(w)
 \le\frac{5n/4+O(D[S\log(e(A+2))+\log(1/\eta)])+K(D+1)}{n+K},
 \qquad w_i=\frac{c_i(n)+1}{n+K}.
\]
Once the pilot absorbs the additive terms, its allocation is affordable.
Inverse-count exploration separately bounds
$G_P(w)=(n+K)\max_{x\in\Omega(P)}\sum_i x_i/(c_i(n)+1)$
by a harmonic-sum argument. The same adaptive history thus provides coverage and low collection cost.

The controller supplies the other part of this argument. With probability
at least $1-\eta$, every prefix up to invocation cap $N_{\rm cap}$ obeys
(Lemma~\ref{lem:mixed-invocation})
\[
 t\le8\sum_{u<t}\frac{b_u(i_u)}{g_u}+a(N_{\rm cap},\eta),\qquad
 a(N_{\rm cap},\eta)=\widetilde O\!\left(D\sum_i\min\{d_i,D\}\right),
\]
where $g_u$ is the optimistic gain scale. For inverse-count rewards,
this bounds time through a harmonic sum and yields coverage. For quota
rewards, it charges productive visits while counting all navigation.
The proof uses the smaller error bound separately on each row; the
learner implements this choice through its two optimistic predictions.

To turn flow cost into collection time, use the unfinished quotas. Flow subadditivity gives
$\mathcal C_P(q)\le3M/2+K(D+1)$ at an affordable allocation.
If $u$ quota units remain, the binary unfinished-row reward has optimal
gain at least $u/\mathcal C_P(q)$. Charging productive visits through a
harmonic sum and using the invocation bound for surplus visits gives
$T_{\rm collect}(M)=\widetilde O(A_{\rm mix}+M)$.
The $1/w_{\min}$ term in Lemma~\ref{lem:main-demand} covers prefix
rounding. This completes the passage from an adaptive history to
requests that can be collected at low interaction cost.

At this point the sampling problem is reduced to choosing the mass $M$.
The remaining task is statistical: a bound expressed through stationary
flows must certify a policy from every initial state. This is where the
auxiliary bias and the transient correction enter.

\subsection{Nonuniform Estimation}\label{sec:nonuniform-estimation}
Stationary mean error must be converted into policy quality from every
starting state. Transient visits need not be flow-balanced, so an
auxiliary bias must control their contribution. The resulting error,
at exact fixed points with the prescribed discount, has the form
\[
 g_P^*-L\ \lesssim\
 \sqrt{\frac{(B_*+1)\alpha G_P(w)}{M}}+\frac{\alpha C_0}{M}.
\]
Thus coverage contributes demand proportional to $G_P(w)(B_*+1)/\epsilon^2$,
while transient effects contribute $C_0/\epsilon$.
Here $\alpha$ is the confidence factor and $C_0$ is the span of a Bellman
bias $z_0$ for auxiliary reward $1/w_i$, whose gain is $G_P(w)$.
It satisfies $C_0\le DG_P(w)$ and is distinct from the target bias span
$B_*$. Write $n_i=q_i=\lceil Mw_i\rceil$ for the statistical sample counts
and $t_i=\alpha/n_i$. The scaled bias $z=(\alpha/M)z_0$ satisfies
\begin{equation}\label{eq:main-nonuniform-witness}
 t_i+P_i z\le\kappa+z(s_i),\qquad
 \kappa=\frac{\alpha G_P(w)}{M},\qquad
 c_z:=\spn(z)=\frac{\alpha C_0}{M}.
\end{equation}
Flow imbalance enters the variance term through the auxiliary span
$c_z$. Value truncation supplies a negative term that controls this
contribution linearly. The cancellation is the second connection.

For the lower endpoint, compare with an optimal policy through a virtual
pessimistic value function capped at height $H\asymp B_*+1$, with
$H\ge\max\{1,4B_*\}$.
Let $\Delta=H-\spn(U_\gamma^*)\ge H/2$, where $U_\gamma^*$ is the
true discounted value of that policy. The key algebraic step is
\begin{equation}\label{eq:main-transient-absorption}
 \sup_{\lambda\ge0}\left\{u\sqrt{Hc_z\lambda}
                -\frac{\Delta\lambda}{2}\right\}
 =\frac{u^2Hc_z}{2\Delta}=O(c_z),
 \qquad \Delta\ge H/2,
\end{equation}
for an absolute constant $u$. The retained negative term absorbs the
square-root cost of flow imbalance, leaving a correction linear in $c_z$.

To obtain these terms, start at a minimum of the capped value and stop
on reaching its cap. Let $b=\mathbb E\gamma^\tau$, $a=1-\gamma$, and
$\lambda=a/[\gamma(1-b)]$. The normalized occupation $d$ before this
stop has a terminal flow imbalance. Applying
\eqref{eq:main-nonuniform-witness} gives
$d^\top t\le\kappa+\lambda c_z$, so the quantile variance bound contributes
$u\sqrt{H(\kappa+\lambda c_z)}$.
The same stopped comparison yields the negative term
$-ab\Delta/(1-b)$. Under the nontrivial-regime conditions in
Appendix~\ref{app:auxiliary-coverage}, the remaining linear terms are
absorbed while retaining $-\Delta\lambda/2$.
Equation~\eqref{eq:main-transient-absorption} then controls the transient
contribution by $O(c_z)$; the discount remainder completes the bound.
The transfer from
virtual caps to the computed policy, including explicit numerical
corrections, is in Theorem~\ref{thm:auxiliary-oracle}.
The upper endpoint needs a separate self-consistency argument bounding
the optimistic value span by $4(B_*+1)$ (Appendix~\ref{app:span-adaptive}).
Two dependencies must be handled before this calculation is legitimate.
First, the prescribed post-pilot row prefixes are defined on independent
row streams before collection, including paths where quotas never finish.
Conditioning on future quota completion would select their values.
Second, the Bellman vectors are fitted to those same prefixes.
Source-state decoupling replaces one state's fixed-point output by a
scalar $u$, and uses a second scalar cap for the other outputs. The
resulting virtual vectors do not use that state's outgoing samples.
A finite grid in the two scalars, including reflected and truncated
vectors, yields simultaneous concentration at the actual fitted fixed
points for both signs and all tested discounts. The virtual caps allow
this concentration argument to apply to the computed fixed points.
Appendix~\ref{app:auxiliary-coverage} gives the construction and event count.

Together these arguments shrink the certificate when
$\kappa\lesssim\epsilon^2/(B_*+1)$ and $c_z\lesssim\epsilon$, giving
sufficient mass
\[
 \widetilde O\!\left(
 \frac{G_P(w)(B_*+1)}{\epsilon^2}+\frac{C_0}{\epsilon}
 +\frac1{w_{\min}}\right).
\]
Thus $C_0\le DG_P(w)$ contributes to the $1/\epsilon$ term without
multiplying the $1/\epsilon^2$ demand.

The two connections now fit together. Coverage controls the leading
statistical term, flow cost controls acquisition, and the cancellation
above confines transient effects to the lower-order accuracy term.
Substituting the pilot guarantees yields the upper bound in
Theorem~\ref{thm:main-total}.

\paragraph{Computation.}
The full schedule, including support queries, rounded backups and
statistical solves, uses
\[
 T_*\,\mathrm{poly}(S,A,L_{\rm in},\log(2+T_*),\log(2/\epsilon))
\]
bit operations on the same success event. Here $L_{\rm in}$ is the input encoding length
(Appendix~\ref{app:computation}).

\section{Lower Bound}\label{sec:lower-overview}
The access-model comparison requires hardness that survives removing
navigation constraints. Theorem~\ref{thm:lower} supplies it with one known
reward and at most two successors per row, even if the learner may query
rows directly. Proving hardness with this stronger access makes the
comparison complete: any continuing learner can be simulated by a
generative learner using one query per transition. A lower bound on the
latter therefore applies to both models. The construction below ensures
that identifying the output policy requires information from many rows,
even when the learner is free to choose those rows.

Put $m=\lfloor S/3\rfloor$, $h=\lceil\log_{A-1}m\rceil$ and
$q=8/\overline D$. A routing tree of height at most $h$ joins $m$
leaf gadgets. Gadget $j$ contains a reward-one state $x_j$ and a
reward-zero state $y_j$. Each of $A-1$ candidate actions at $x_j$ moves
to $y_j$ with probability $p_{jl}$; $y_j$ returns with probability $q$.
An extra action leaves the gadget through the tree. The base kernel has
$p_{jl}=q$. Alternative $(j,l)$ changes only that row to
$p_{jl}=q(1-8\epsilon)$, keeping rewards and support identical.

The changed gadget has gain $1/(2-8\epsilon)>1/2+2\epsilon$;
all other recurrent choices have gain at most $1/2$. All-state
$\epsilon$-optimality therefore requires the output policy to route to
the changed gadget and select its changed action. This produces disjoint
output events $E_{jl}$ under the common support, avoiding any claim that
actions in merely transient states must be individually optimal.
Controlled routes take at most
$1/q+2h+1/[q(1-8\epsilon)]\le\overline D$ in expectation, and the
construction uses at most $3m-1\le S$ states.

For a learner truncated at the common budget $T$, the adaptive KL chain
rule gives
\[
 \mathrm{KL}(\mathbb P_0^T\Vert\mathbb P_{jl}^T)
 =\mathbb E_0N_{jl}(T)\,\mathrm{kl}(q,q(1-8\epsilon)),
 \qquad \mathrm{kl}(q,q(1-8\epsilon))=
 O(\epsilon^2/\overline D).
\]
At least $m(A-1)-2$ disjoint events have base probability at most $1/2$;
correctness makes their respective alternative probabilities at least
$1-\delta$. Binary data processing then requires KL at least
$\Omega(\log(1/\delta))$. Summing the required row counts gives
\[
 T=\Omega\!\left(\frac{SA\overline D}{\epsilon^2}\log(1/\delta)\right).
\]
The diameter factor persists even with generative access because each
queried sample carries only $O(\epsilon^2/\overline D)$ information
about a change that affects long-run performance. On this family,
diameter therefore measures statistical difficulty as well as navigation
cost. The upper bound shows that, in the matching range, controlled
collection fits within this unavoidable statistical budget.

The conditions in Section~\ref{sec:matching} ensure all constants and
routing bounds used here. Appendix~\ref{app:lower-proof} includes the
truncation argument and complete proof. Together with the upper bound,
this establishes the stated matching regimes. Matching in the dense coarse-accuracy range remains open.

\section{Discussion and Conclusion}\label{sec:scaling}
The matching result shows that navigation constraints can be absorbed
within the statistical budget for the model classes and accuracy ranges
of Corollary~\ref{cor:access-matching}. The learner discovers which
requests the environment can supply, pays for acquiring them, and
certifies their sufficiency using observed data. All three operations
are covered by one finite interaction guarantee.

\paragraph{Numerical illustration.}
Appendix~\ref{app:additional-results} illustrates the statistical bound's
separation of bias span from a comparison policy's hitting time at fixed
sample sizes. This experiment uses uniform allocation, a special case
also covered by prior generative analyses~\citep{zc}, to isolate the
fixed-prefix certificate. Allocation learning and the adaptive schedule
are analyzed in the theoretical guarantees. The access-model matching result rests on the upper and
lower bounds above.

The mathematical contribution is the connection between these operations.
An adaptive history yields a low-cost balanced flow even before stationary
behavior emerges. Nonuniform estimation then preserves the statistical
rate by controlling transient error linearly in inverse sample mass.
Together, these facts explain how one continuing trajectory reaches the
generative worst-case rate in the stated regimes.

For unrestricted support at coarse accuracy, the startup upper bound can
still dominate. Resolving that range requires a sharper understanding of
the interactions needed to learn a useful allocation. A useful next step is to determine how accurately the allocation must
be learned at each target policy accuracy. This would clarify when less
exploration suffices for certification, sharpening startup while
preserving the established precision dependence. Our bounds isolate this question and remove the extra leading
diameter factor throughout the matching range.

\clearpage
\bibliographystyle{plainnat}
\bibliography{references}

@inproceedings{zzc,
  author = {Matthew Zurek and Guy Zamir and Yudong Chen},
  year = {2025},
  title = {{Optimal Single-Policy Sample Complexity and Transient Coverage for Average-Reward Offline RL}},
  booktitle = {Advances in Neural Information Processing Systems},
  volume = {38},
  pages = {187213--187277},
}

@inproceedings{tdk,
  author = {Adrienne Tuynman and R\'emy Degenne and Emilie Kaufmann},
  year = {2024},
  title = {{Finding Good Policies in Average-Reward Markov Decision Processes Without Prior Knowledge}},
  booktitle = {Advances in Neural Information Processing Systems},
  volume = {37},
  pages = {109948--109979},
}

@inproceedings{zc,
  author = {Matthew Zurek and Yudong Chen},
  year = {2025},
  title = {{Span-Agnostic Optimal Sample Complexity and Oracle Inequalities for Average-Reward RL}},
  booktitle = {Proceedings of the 38th Conference on Learning Theory},
  volume = {291},
  pages = {6156--6209},
  series = {Proceedings of Machine Learning Research},
}

@inproceedings{navigation,
  author = {Al Marjani, Aymen and Aur\'elien Garivier and Alexandre Proutiere},
  year = {2021},
  title = {{Navigating to the Best Policy in Markov Decision Processes}},
  booktitle = {Advances in Neural Information Processing Systems},
  volume = {34},
  pages = {25852--25864},
}

@inproceedings{gosprl,
  author = {Jean Tarbouriech and Matteo Pirotta and Michal Valko and Alessandro Lazaric},
  year = {2021},
  title = {{A Provably Efficient Sample Collection Strategy for Reinforcement Learning}},
  booktitle = {Advances in Neural Information Processing Systems},
  volume = {34},
}

@misc{lee,
  author = {Jongmin Lee and Ernest K. Ryu and Vaneet Aggarwal},
  year = {2026},
  title = {{Learning Policy from a Single Trajectory in Average-Reward Markov Decision Process}},
  eprint = {2606.16729},
  howpublished = {arXiv preprint arXiv:2606.16729},
  archivePrefix = {arXiv},
}

@misc{wwy,
  author = {Jinghan Wang and Mengdi Wang and Lin F. Yang},
  year = {2022},
  title = {{Near Sample-Optimal Reduction-based Policy Learning for Average Reward MDP}},
  eprint = {2212.00603},
  howpublished = {arXiv preprint arXiv:2212.00603},
  archivePrefix = {arXiv},
}

@misc{lrp,
  author = {Joseph Lazzaro and Alessio Russo and Aldo Pacchiano},
  year = {2026},
  title = {{Non-Asymptotic Best Policy Identification Guarantees in Online Reinforcement Learning}},
  eprint = {2607.17201},
  howpublished = {arXiv preprint arXiv:2607.17201},
  archivePrefix = {arXiv},
}

@misc{jiao,
  author = {Yuchen Jiao and Jiin Woo and Gen Li and Gauri Joshi and Yuejie Chi},
  year = {2026},
  title = {{Sample Complexity of Average-Reward Q-Learning: From Single-agent to Federated Reinforcement Learning}},
  eprint = {2601.13642},
  howpublished = {arXiv preprint arXiv:2601.13642},
  archivePrefix = {arXiv},
}

@inproceedings{fqi,
  author = {Jongmin Lee and Ernest K. Ryu},
  year = {2025},
  title = {{Finite-Time Bounds for Average-Reward Fitted Q-Iteration}},
  booktitle = {Advances in Neural Information Processing Systems},
  volume = {38},
  pages = {132144--132177},
}

@inproceedings{zcplugin,
  author = {Matthew Zurek and Yudong Chen},
  year = {2025},
  title = {{The Plug-in Approach for Average-Reward and Discounted MDPs: Optimal Sample Complexity Analysis}},
  booktitle = {Proceedings of the 36th International Conference on Algorithmic Learning Theory},
  volume = {272},
  pages = {1386--1387},
  series = {Proceedings of Machine Learning Research},
}

@inproceedings{focus26,
  author = {Guy Zamir and Matthew Zurek and Yudong Chen},
  year = {2026},
  title = {{Optimal Variance-Dependent Regret Bounds for Infinite-Horizon MDPs}},
  booktitle = {Proceedings of the 39th Conference on Learning Theory},
  volume = {336},
  pages = {7016--7061},
  series = {Proceedings of Machine Learning Research},
}

\clearpage
\appendix
\numberwithin{theorem}{section}
\numberwithin{equation}{section}
\numberwithin{figure}{section}
\numberwithin{table}{section}
\section{Additional Preliminaries}\label{app:preliminaries}
\subsection{Known-model occupation design}\label{app:static-design}
If $P$ were known, a positive maximizer of $\sum_i\log w_i$ over
$\Omega(P)$ would satisfy $G_P(w)=K$ and $\mathcal C_P(w)=1$.
Communication supplies a positive feasible flow by mixing regenerative
cycles through all rows. The log objective tends to minus infinity at
the boundary; its first-order condition gives
$\sum_i x_i/w_i\le K$ for every feasible $x$, with equality at $x=w$.
This known-model observation motivates learning adequate proportions from the current physical
state with unknown $P$ and no structural inputs, at the stated startup
cost, and connect them to reliable observable stopping.
The pilot uses auxiliary rewards $b_t(i)=1/(c_i+1)$ to encourage coverage.
The policy guarantee concerns the target reward $r$, fixed and known
before data collection.

The symbols $c_i$, $k_i$, $q_i$, $M$ and $N$ denote pilot counts,
collector counts, requested row-prefix lengths, nominal mass and total
requested samples, respectively. $B_*$ concerns the target reward;
$C_0$ concerns auxiliary reward $1/w_i$. Only the analysis uses $D,d_i,G_P,\mathcal C_P,B_*,C_0$.

\subsection{Total-support bound}
Since $\sum_i\min\{d_i,D\}\le\min\{J_{\rm supp},SAD\}$,
the diameter bound~\eqref{eq:diameter-bound} also yields
\[
\widetilde O(D\min\{J_{\rm supp},SAD\}+SAD/\epsilon^2).
\]
This weaker expression follows for the same algorithm.

\subsection{Two-state allocation identities}\label{app:two-state-check}
For the two-state example in Section~\ref{sec:allocation}, every stationary flow has state
marginal $\mu=(1,p)/(1+p)$. Uniform row weights give
$G_P=4\sum_s\mu_s=4$. Equal splitting within each state gives
$\sum_{s,a}x_{sa}/(\mu_s/2)=\sum_s2=4$ as well.
A balanced flow of mass $C$ must supply total rare-state quota $1/2$,
so $C\mu_1\ge1/2$. Equal action splitting attains
$C=(1+p)/(2p)$ for uniform requests. The occupation allocation is itself a
balanced unit flow, and total requested mass one implies $C\ge1$.

\section{Algorithm Details}\label{app:algorithm-details}
This appendix specifies the controller and statistic used in
Algorithm~\ref{alg:main}. Their probability and progress arguments appear in
Appendix~\ref{app:upper-proofs}; bit computation is in Appendix~\ref{app:computation}.
\subsection{Rounded exploration controller}\label{app:rounded-controller}
Consider a finite communicating transition kernel $P$ with $S\ge2$ states,
$A\ge1$ actions, $K=SA$ rows and diameter $D$. Neither $P$ nor $D$ is an
input. Fix an interaction cap $N\ge1$ and confidence $\delta\in(0,1)$.
Rewards $b_t(i)$ are either inverse counts $1/(n_t(i)+1)$ or indicators
of unfinished integer row quotas. They decrease coordinatewise.
Neither physical state nor invocation counts are reset.

At scale $d=2^j$, put $B=2d$, $\alpha=1/B$, $\gamma=1-\alpha$ and
$\theta=[256B(N+1)]^{-1}$. Initialize $V=B\mathbf1$. At gain-epoch
entry set $g=\max_s V(s)/B$, clear scalar buffers, and use
\[
 \delta_{j,h}=\frac{\delta}{12(j+1)^2(h+1)^2},\quad
 L_{\rm sc}=\left\lceil\log_2\frac{100K(N+1)^6}{\delta_{j,h}}\right\rceil,
 \qquad L=L_{\rm sc},\quad W=\lceil2^{20}BL\rceil.
\]
The index $h$ starts at zero on each scale. Each row stores the values
$Y_t=V_t(s_{t+1})/g$ computed using the vector chosen \emph{before}
that transition. Old observations are never reevaluated using a new vector.
For an unobserved row predict $U_i=B$. Otherwise put
$\ell_i=\min\{W,n_i^{\rm epoch}\}$ and use its last $\ell_i$ scalars,
mean $\bar Y_i$, and centered square sum $Q_i$ to compute
\[
 U_i=\min\{B,\bar Y_i+4\sqrt{Q_iL}/\ell_i+32BL/\ell_i\},\qquad
 U_i^+=\min\{B,\lceil256U_i\rceil/256\}.
\]
\paragraph{Order of a decision.}
At physical time $t$, before drawing a successor, let $V^-$ be the
current vector. Use the counts and buffers from transitions $u<t$ to
compute all row predictions, then form all scores
$q_{sa}=b_t(s,a)+\gamma gU^+_{sa}$. With these scores fixed, update
\emph{every state synchronously}:
\[
 V^+(s)=\min\{V^-(s),\theta\lceil\max_a q_{sa}/\theta\rceil\}.
\]
If $\min V^+<\max V^+/2$, double $d$ and reinitialize the scale.
Otherwise, if $\max V^+/B<g/2$, retain $V^+$ and start a new gain epoch.
After either restart, recompute all predictions and scores before testing
again; $t$, the physical state and invocation counts do not change.
If neither test fires, set $V_t=V^+$ and choose an action maximizing
the already computed \emph{unrounded} $q_{s_ta}$, breaking ties by action
index. Draw $s_{t+1}$, append $V_t(s_{t+1})/g$ to that row's buffer,
and update counts (and row regions in the combined controller).
Only then advance to $t+1$ and the new auxiliary rewards.
Thus the stored direction is the post-backup vector, fixed before the
transition. Stop at the
cap or immediately upon quota completion. Scalar buffers contain only
observations from the current gain epoch.

\paragraph{Scalar confidence bounds.}
For predictable coordinatewise decreasing vectors $f_t\in[0,B]^S$
within an epoch, the unrounded predictions above satisfy simultaneously
at queried rows, up to the cap,
\begin{align}
 U_t&\ge P_{i_t}f_t,\label{rc:eq:confidence}\\
 \sum_{t<T}(U_t-P_{i_t}f_t)
 &\le3KBW+10\sqrt{T\mathcal V L/W}+64BLT/W,
 \qquad \mathcal V=\sum_{t<T}\Var_{P_{i_t}}(f_t).
 \label{rc:eq:window}
\end{align}
Its failure allowance is $\delta_{j,h}$.
Its lag accounting charges each nonnegative mean decrease to at most
$W$ later queries; the total is at most $KBW$.
Rounding adds at most $T/256$ to the right side of
\eqref{rc:eq:window} and preserves \eqref{rc:eq:confidence}.

\subsection{Combined controller}\label{app:mixed-spec}
\paragraph{Exact controller.}
Use the controller of Appendix~\ref{app:rounded-controller}, with either
inverse-count or unfinished-quota rewards. Give its scalar construction
allowance $\eta/2$, and maintain the dyadic nested row regions of
Appendix~\ref{app:row-support} at allowance $\eta/2$ throughout each
invocation. At every gain-epoch entry, including entry after a scale
restart, replace the scalar-only logarithm by
\[
 L=\max\left\{L_{\rm sc},\;
 2+\left\lceil\log_2\frac{2^{15}KS(N+1)^4}{\eta}\right\rceil\right\},
 \qquad W=2^{20}BL,
\]
where $L_{\rm sc}$ uses $\delta=\eta/2$ in
Appendix~\ref{app:rounded-controller}. This $L$ is used in both prediction radii and the window length. It dominates the row-confidence logarithms for all retained
prefixes $n\le N$. Before every backup, including zero-transition restart loops,
intersect its predictor with the upward $1/256$ rounding of
$\max_{q\in\mathcal C_i}q(V/g)$. For this support prediction use $V^-/g$ from the decision order above;
all predictions are fixed before the simultaneous update to $V^+$.
Within an epoch, intersect retained upper predictions with
the current support prediction: older vectors dominate the current one,
so these retained bounds remain optimistic. A new observation refreshes
its row's scalar prediction, and the next backup intersects it again.
Keep the nested regions through gain and scale restarts. An empty row region off containment contributes no
coordinate restriction. Stop when quotas are met.

\paragraph{Invocation and data boundaries.}
The pilot and quota collector are distinct instances of this rule.
The pilot starts with zero counts $c_i$ and freezes
$w_i=(c_i+1)/(2^p+K)$ after its last transition. A new quota instance starts
at the resulting physical state, with zero invocation counts $k_i$,
initial simplex row regions and empty scalar buffers, and targets
$q_i=\lceil Mw_i\rceil$. Its reward is $\mathbf1\{k_i<q_i\}$.
Every visit increments $k_i$ and updates the row confidence information,
including surplus visits. Only the first $q_i$ post-pilot successors
enter the statistical histogram. Within an invocation, gain/scale
restarts retain $k_i$ and the nested regions, but discard scalar windows;
a new attempt discards all these data. Thus neither pilot counts nor
surplus tails are silently substituted for the certificate's fixed prefixes.

\subsection{Statistical subroutine}
\label{app:nonuniform-statistic}
This specification collects the statistical definitions used by Algorithm 1
in one place. Its inputs are the known reward, accuracy $\epsilon$, a
failure allowance $\eta$, an integer mass $M$, positive rational weights
$w$ frozen at a finite pilot history, and the first $n_i=\lceil Mw_i\rceil$
fresh successors of each row. The collection must be complete before
the computation is called. Put $N=\sum_i n_i$; $M$ and $N$ generally differ.
The outer algorithm skips $M<2/w_{\min}$, so all invoked counts are
positive and $N\le3M/2$.

\paragraph{Parameters and row functionals.}
Set
\begin{gather*}
 L_g=\lceil\log_2(M+1)\rceil,\quad
 L_h=\lceil\log_2(4M)\rceil,\quad a_{\min}=2^{-L_g},\quad
 J=\lceil128N/a_{\min}^2\rceil,\\
 \alpha=16\left\lceil\log_2
 \frac{64K(S+1)((J+1)^2+1)(L_h+1)(L_g+2)}{\eta}\right\rceil,
 \qquad \beta_i=\frac{\alpha}{\max(n_i-1,1)}.
\end{gather*}
For a vector $v$, when $\beta_i\le1$, let $\xi_i$ be the largest coordinate
value with $\widehat P_i\{v\ge \xi_i\}\ge\beta_i$ and put
$z_j=\min(v_j,\xi_i)$. When $\beta_i>1$, put
$z=(\min v)\mathbf1$. Define
\[
 F_i^-(v)=\max\!\left\{\widehat P_i z-
 \max\{\sqrt{\beta_i\operatorname{Var}_{\widehat P_i}(z)},
                  \beta_i\spn(z)\}-5/N,\ \min v\right\},
 \qquad F_i^+(v)=-F_i^-(-v).
\]
The quantile monotonicity result of \citet[Lemma~B.2 and Appendix~E.1]{zzc}
applies to this penalty, minimum floor and constant subtraction.
Its rowwise argument allows any fixed $\beta_i>0$, including the
minimum-vector branch when $\beta_i>1$. Reflection gives the upper sign.
Monotonicity and constant-shift equivariance imply sup-norm
nonexpansiveness. The range lies between $\min v$ and $\max v$;
reflection preserves these properties. Consequently the discounted
Bellman operators below are contractions on $[0,1/(1-\gamma)]^S$.
Minima and spans include zero-mass coordinates. These functionals are the functionals used in Appendix~\ref{app:auxiliary-coverage}, with
the enlarged common confidence parameter for both signs and all discounts.

\paragraph{Fixed points, output, and numerical acceptance.}
For $T_\gamma^\pm v(s)=\max_a\{r_{sa}+\gamma F_{sa}^\pm(v)\}$,
the lower discount is $\gamma_-=1-1/(M+1)$ and the upper discounts are
$\gamma_j=1-2^{-j}$, $1\le j\le L_g$.
At exact fixed points, use the lower greedy policy and endpoints
$(1-\gamma_-)\min V^-$ and
$\min_j(1-\gamma_j)\max V_j^+$.

For finite computation use outward enclosures. For any discount write
$a=1-\gamma$. Enclose $T_\gamma h-h$ by $[l_s,u_s]$, and set
$c=(\min l+\max u)/2$, $R=(\max u-\min l)/2$.
The shifted vector $W=h+c/a$ is within $\nu=R/a$ of the exact fixed point.
For the lower solve choose the action maximizing
$r_{sa}+\gamma_-F_{sa,\rm lo}^-(W)$, and use
\[
 L=\max\left\{0,\min_s\bigl[
 r_{s\widehat\pi(s)}+F_{s\widehat\pi(s),\rm lo}^-(W)-W_s\bigr]
 -2\nu\right\}.
\]
For each upper solve set $U_j=c+a\max h+R$ and
$U=\min\{1,\min_j U_j\}$. Return the lower policy only if
$U-L\le\epsilon$. Enclose rows to width at most $\epsilon/128$ and
obtain $\nu\le\epsilon/128$ using the complete rounded contraction
fallback of Appendix~\ref{app:auxiliary-coverage} when needed.

Clipping the Bellman output to span $\lceil1/a\rceil$ is inactive on
the invariant cube. It therefore preserves the exact fixed point, and
the resulting monotone contraction satisfies the same residual-to-distance
bound. Constant-shift equivariance also ensures that row evaluations at
$h$ and at $W=h+c/a$ give the same policy and undiscounted residuals.
Thus the residual certificate applies to both evaluations.

The lower residual correction follows from that appendix's policy
extraction argument, and upper residual correction follows by contraction.
Source-state decoupling and the union over signs and discounts in
Appendix~\ref{app:span-adaptive} establish simultaneous soundness on
virtual fixed row prefixes before intersecting their actual availability.
This applies conditional on any finite pilot history with positive frozen
$w$; no coverage, diameter certificate, or true model is an input.
The sufficient mass is then established in that appendix, separately
from this soundness statement. The uniform experiments specialize to
$n_i=n$, $M=N=Kn$.

\paragraph{All-state validity, including multichain outputs.}
Let $L_{\rm raw}$ be the corrected lower endpoint before taking its
maximum with zero. It implies
$r_{s\widehat\pi(s)}+P_{s\widehat\pi(s)}W-W_s\ge L_{\rm raw}$
for every state.
Summing its conditional expectations for $t$ transitions gives
$\mathbb E_s\sum_{u<t}r_u\ge tL_{\rm raw}+W_s-\mathbb E_sW_{s_t}$.
The bounded endpoint term vanishes after division by $t$, so
$g_P^{\widehat\pi}(s)\ge L_{\rm raw}$ for every starting state.
Nonnegative rewards also give $g_P^{\widehat\pi}(s)\ge0$, hence the
reported bound $L=\max\{0,L_{\rm raw}\}$. For an exact upper fixed point $V$,
all actions satisfy $r_i+\gamma P_iV\le V_s$. Hence
$r_i+P_iV-V_s\le(1-\gamma)\max V$; the same telescope bounds every
policy's gain. The residual enclosure enlarges this to $U_j$.
Communication of the controlled MDP guarantees a constant optimal gain.

\subsection{Schedule and confidence}\label{app:outer-schedule}
At outer stage $j$, let $H_j=2^j$. Enumerate $p,m\in\{0,\ldots,j\}$
lexicographically. Each attempt starts fresh algorithm memory at the
current physical state, runs a $2^p$-step pilot, and freezes $w$.
For $M=2^m<2/w_{\min}$ skip after paying the pilot. Otherwise collect
the fresh prefixes with cap $H_j$, and run statistics only after completion.
Return only when the corrected width is at most $\epsilon$; else continue.
The environment is never reset. Each attempt costs at most $2H_j$.
Assign each of the pilot-controller, trajectory flow balance, quota-controller and
statistical events $\eta_j=\delta/[64(j+1)^4]$.

The controller and statistic above are invoked with these event allowances. The joint probability and stopping-cost calculation is in Appendix~\ref{app:rowwise-mixed}.

\section{Proofs of the Upper Bounds}\label{app:upper-proofs}
The overall interaction bound is proved in Appendix~\ref{app:rowwise-mixed}.
Its ingredients can be read separately: trajectory flow balance and pilot
quality in Sections C.3--C.4, statistical guarantees in C.5--C.7, and
controller bounds in C.1--C.2. Appendix~\ref{app:computation} establishes
bit complexity.
\subsection{Row confidence sets}\label{app:row-support}
Fix a row-family failure allowance $\delta_r\in(0,1)$.
These are the confidence regions used by the combined controller, independent
of any coverage test or occupation-design algorithm. For a row with
$n\ge1$ observations and empirical probabilities $\widehat p_j$, let
$\ell_n=\lceil\log_2(32KS(n+1)^2/\delta_r)\rceil$.
The coordinate interval contains exactly those $q_j\in[0,1]$ satisfying
\[
 |q_j-\widehat p_j|\le\sqrt{2\ell_n q_j/n}+2\ell_n/(3n),
\]
with endpoints rounded outward by at most $2^{-16}/(n+1)^2$.
Bernstein's inequality at the true coordinate, followed by a union over
coordinates, rows and prefixes, gives failure at most $\delta_r/2$.
At $n=0$, use the full probability simplex.

\paragraph{Confidence sets and transition-value error.}
Write $K=SA$ and let $\delta_r$ be the row-family failure allowance. For a row with $n\ge1$
observations, let $B$ be its observed support, $k=|B|$, and
$U=\mathcal S\setminus B$. Intersect the coordinate intervals defined above, whose total failure
allowance is already $\delta_r/2$, with
\[
 q(U)\le\min\{1,\beta/n\},\qquad
 \beta=\left\lceil\log_2
 \frac{2KS n(n+1)\binom Sk}{\delta_r}\right\rceil.
\]
Use the simplex when $n=0$. For a fixed subset $B$ of size $k$, observing all $n$ successors in
$B$ has probability at most $\exp[-nP(B^c)]$. Union over the
$\binom Sk$ subsets, $k=1,\ldots,S$, rows and $n\ge1$ gives unseen-mass
failure at most $\delta_r/2$, since $\sum_{n\ge1}1/[n(n+1)]=1$.
The coordinate allowance supplies the other $\delta_r/2$.
Thus containment holds simultaneously at all row prefixes.
The empirical row is feasible. All finite-prefix coordinate bounds have
strict outward slack and the group cap is positive, so a sufficiently
small mixture with the uniform law is feasible. Thus the extended process
communicates even on histories where true containment fails.

We record the quantitative transport argument, including rounding.
For $L\ge\lceil\log_2(32KS(n+1)^2/\delta_r)\rceil$, put $a=L/n$.
Use coordinate endpoint precision
$e\le2^{-16}/(n+1)^2$. The interval's defining Bernstein inequality and
$\sqrt{2ae}+e\le a/3$ imply
$|q_j-\widehat p_j|\le\sqrt{2a q_j}+a$ for any feasible $q$.
For two feasible rows $p,q$, writing $u_j=|q_j-p_j|$ gives
\[
 u_j\le2\sqrt{2ap_j}+\sqrt{2au_j}+2a
      \le2\sqrt{2ap_j}+u_j/2+3a,
 \qquad u_j\le4\sqrt{2ap_j}+6a.
\]
Also $\beta\le2kL$, since
$\log_2\binom Sk\le k\log_2 S\le kL$ and the remaining log is at
most $L$. Center any real vector $v$ at $pv$. On $B$,
Cauchy--Schwarz bounds the error by
$4\sqrt{2ka\Var_p(v)}+6ka\spn(v)$; on $U$ it is at most
$(p(U)+q(U))\spn(v)\le4ka\spn(v)$. As $n+1\le2n$,
\begin{equation}\label{eq:sparse-transport}
 |(q-p)v|\le8\sqrt{\frac{kL\Var_p(v)}{n+1}}
                  +\frac{20kL\spn(v)}{n+1}.
\end{equation}
The range bound handles $n=0$ with effective support one. This is uniform
over all $v$ and all compatible rows, including rows with positive mass
outside $B$.

\paragraph{Dyadic confidence sets.}
Retain the above region at every positive power-of-two row count and
intersect all retained regions. Between these counts keep the previous
intersection. Define
\[
 a_i(n)=\min\bigl(\{1\}\cup
 \{k_i(m)/(m+1):m\le n,\ m\text{ a power of two}\}\bigr).
\]
Containment is inherited from the all-prefix event; no additional union
is needed for the data-dependent intersection. The coefficient is
nonincreasing and, for $n\ge1$, the largest dyadic witness
$m\le n<2m$ gives $a_i(n)\le2d_i/(n+1)$.
At a witness attaining the minimum, both the true row and any compatible
row satisfy \eqref{eq:sparse-transport}. Thus for an envelope $L$ covering
the retained prefixes,
\[
 |(q-P_i)v|\le8\sqrt{La_i(n)\operatorname{Var}_{P_i}(v)}
                  +20La_i(n)\operatorname{sp}(v).
\]
The $n=0$ case uses the range bound. This inequality is uniform over
$v$, so subsequent value-vector updates need no new confidence event.
Merge coordinate lower/upper bounds by maxima/minima and caps on the
same unseen set by minima. Unseen sets form a nested chain, so there are
at most $S+1$ distinct group constraints. Sorting values and greedily
filling the lower-bound-contracted laminar capacities solves each support
query. Empty intersections off the containment event contribute no
coordinate restriction.

\subsection{Exploration Bounds}\label{app:exploration-proof}
\subsubsection{Scalar prediction bound}
We condition only on the finite history at epoch entry. Thus $B,W$ and
the confidence allowance $\eta=\delta_{j,h}$ are fixed. The following
argument applies to any predictable decreasing vectors, without assuming
that controller confidence or a later test succeeds. Write
$\mu_t=P_{i_t}f_t$ and $v_t=\Var_{P_{i_t}}(f_t)$.

\begin{lemma}[A masked martingale bound]
Let $X_t=I_t\widetilde X_t$, where $\widetilde X_t$ are martingale
differences with $|\widetilde X_t|\le b$, $b>0$, and $I_t\in\{0,1\}$
is predictable. Assume $\sum_t I_t\le N$ almost surely for a deterministic
integer $N\ge1$. Put $Z_t=\sum_{u<t}X_u$ and
$V_t=\sum_{u<t}\mathbb E_u X_u^2$.
For any $L\ge1$, except on an
event of probability at most $(N+1)e^{-L}$, simultaneously in $t$,
\[
 Z_t\le 2\sqrt{V_tL}+4bL.
\]
The same assertion holds for $-Z_t$, with its own failure allowance.
\end{lemma}
\begin{proof}
For $0<\lambda<3/b$, the Taylor series and
$k!\ge2\,3^{k-2}$ for $k\ge2$ imply
\[
 \mathbb E_t e^{\lambda X_t}
 \le \exp\!\left(\frac{\lambda^2\mathbb E_tX_t^2}
                         {2(1-\lambda b/3)}\right).
\]
The product supermartingale therefore gives
$Z_t\le\lambda V_t/[2(1-\lambda b/3)]+L/\lambda$
uniformly in time, outside probability $e^{-L}$.
For the uniform claim, stop the nonnegative supermartingale at its first
crossing of $e^L$. Apply its expectation bound to bounded stopping times,
then pass to the limit.

For a deterministic upper variance $v>0$, choose
$\lambda=\sqrt{2L/v}/(1+(b/3)\sqrt{2L/v})$.
On $V_t\le v$ the resulting boundary is
$\sqrt{2vL}+bL/3$. Use $v=b^2$ for the bottom bin and
$v=2^kb^2$ for bins $2^{k-1}b^2<V_t\le2^kb^2$.
Predictability and the activation budget give
\[
\begin{aligned}
 V_t&=\sum_{u<t}\mathbb E_u[(I_u\widetilde X_u)^2]
     =\sum_{u<t}I_u\mathbb E_u\widetilde X_u^2\\
    &\le b^2\sum_{u<t}I_u\le Nb^2.
\end{aligned}
\]
Hence at most $1+\lceil\log_2 N\rceil\le N+1$ bins suffice. In a nonbottom bin $v\le2V_t$. In the bottom bin,
$b\sqrt{2L}+bL/3\le4bL$. Union over bins proves the claim.
\end{proof}

Fix a deterministic row $i$, a deterministic row-start index $a$, and
a deterministic length
\[
 w\in\{1,\ldots,\min\{W,N\}\}.
\]
Let $n_i(t)$ count earlier visits within the epoch and use the mask
\[
 I_t=\mathbf1\{i_t=i,\ a\le n_i(t)<a+w\}.
\]
The mask, direction and conditional mean are chosen before the next
transition. Each activation increments the row count, so the mask
activates at most $w\le N$ times, including activations whose centered
increment equals zero. Apply the preceding lemma to the masked centered increments
of $Y_t$, $(Y_t-\mu_t)^2$, and $(Y_t-c)^2$ for each fixed
$c$ in the grid $\{jB/(N+1)^2:0\le j\le(N+1)^2\}$.
Use both tails for $Y_t$, the upper tail for the squared residual,
and the lower tail for each fixed-center square. The masks remain
defined when an interval is abandoned; no event conditions on its
completion. At completion the sums are precisely those of that row window.

For one such window, put $V_w=\sum v_t$, $Z_w=\sum(Y_t-\mu_t)$,
$Q_w=\sum(Y_t-\bar Y)^2$, and $A_w=\sum(Y_t-\mu_t)^2$.
The centered-square increments have conditional variance at most
$B^2v_t$. Thus the simultaneous bounds imply
\begin{align}
 |Z_w|&\le2\sqrt{V_wL}+4BL,\label{rc:eq:localZ}\\
 A_w&\le V_w+2B\sqrt{V_wL}+4B^2L
       \le2V_w+5B^2L\le2V_w+8B^2L.\label{rc:eq:localA}
\end{align}
For $M_c=\sum\mathbb E_t(Y_t-c)^2$, the fixed-center bound gives
\[
 M_c\le\sum(Y_t-c)^2+2B\sqrt{M_cL}+4B^2L
 \le\sum(Y_t-c)^2+M_c/2+6B^2L.
\]
Here the centered increments are bounded by $B^2$ and their cumulative
conditional variance is at most $B^2M_c$. Also $V_w\le M_c$.
Select a grid center within $B/(N+1)^2$ of $\bar Y$ only \emph{after}
the event has been made uniform over grid centers. The identity
$\sum(Y_t-c)^2=Q_w+w(c-\bar Y)^2$ and $w\le N$ for a completed
window give
\begin{equation}
 V_w\le2Q_w+14B^2L\le2Q_w+20B^2L.\label{rc:eq:localQ}
\end{equation}
The argument also applies to partial windows when $W>N$.

To verify the probability count, there are at most $K(N+1)^2$ masks,
at most $(N+1)^2+4\le3(N+1)^2$ signed processes per mask, and at most
$N+1$ bins per process. Thus total failure is at most
$3K(N+1)^5e^{-L}\le\eta$ for the unchanged chosen $L$.
There is no additional union over query times because each bound is
uniform in global time. Conditioning on epoch entry fixes all grids;
ending an epoch simply stops its predictable masked processes.

From \eqref{rc:eq:localZ} and \eqref{rc:eq:localQ},
$|Z_w|\le3\sqrt{Q_wL}+13BL$.
The current row mean is at most the average of the historical means,
so the prediction radius, using the actual length $w$, proves
\eqref{rc:eq:confidence} for partial and full windows. This remains
true at intervening backup queries with no new observation of the row:
the historical window is unchanged and the current mean can only decrease.

For the cumulative bound, count only the actual selected-row queries.
At a full-window query let $\mu$ be the current mean and
$d=W^{-1}\sum(\mu_t-\mu)\ge0$. All summands lie in $[0,B]$.
By \eqref{rc:eq:localA},
\[
 Q_w\le\sum(Y_t-\mu)^2
 \le2A_w+2\sum(\mu_t-\mu)^2
 \le4V_w+16B^2L+2BWd.
\]
Combine this with \eqref{rc:eq:localZ}, the prediction formula, and
subadditivity of square roots to obtain
\begin{align*}
 U-\mu&\le d+10\sqrt{V_wL}/W+52BL/W
                         +4\sqrt{2BLd/W}\\
       &\le2d+10\sqrt{V_wL}/W+60BL/W.
\end{align*}
Clipping at $B$ can only decrease this upper bound. For each row,
write its successive mean drops as nonnegative increments. Each drop
occurs in at most $W$ later window lags, with coefficient at most one;
their total is therefore at most $WB$. There are at most $KW$ warm-up
queries, each with excess at most $B$. Each observed conditional variance
belongs to at most $W$ later windows, so Cauchy--Schwarz gives
\[
 \sum_{\text{full queries}}\sqrt{V_w}
 \le\sqrt{TW\mathcal V}.
\]
Summing proves the slightly stronger coefficient $60$ in the last
term of \eqref{rc:eq:window}, and hence its displayed coefficient $64$.
This proves both statistical assertions used by the controller.

\subsubsection{Invariants and restart costs}
\begin{lemma}
On the simultaneous scalar-confidence event, every updated $V$ is a
supersolution of the current true discounted Bellman operator. No scale
larger than the first dyadic $d\ge D$ is reached. At an executed step,
$g/2\le\max V/B\le g$, $V/g\in[B/4,B]^S$, and $\theta/g\le1/256$.
\end{lemma}
\begin{proof}
Initialization is a supersolution since $b\le1$. Historical scalar
directions dominate the current pre- and post-update vectors, since $V$
decreases within an epoch. Thus $U^+\ge P(V_{\rm new}/g)$.
Both $V_{\rm old}$ and the rounded backup dominate
$T_{P,b}V_{\rm new}$: for the former use the previous supersolution,
decreasing rewards and monotonicity, and for the latter use confidence.
Their coordinatewise minimum is therefore a supersolution. A gain
restart retains this invariant; a scale restart reinitializes it.

For any target $z$, follow a hitting policy with mean time at most $D$.
Nonnegative rewards and the supersolution imply
$V(s)\ge\mathbb E[\gamma^{\tau_z}]V(z)$ by stopping at
$\tau_z\wedge n$ and discarding the nonnegative unhit term.
Since $\gamma^m\ge1-\alpha m$,
$\min V\ge(1-\alpha D)\max V$. At $d\ge D$ the scale test cannot fire.
Consequently $d<2\max\{1,D\}$ on this event.

Both update terms are at least $\max_a b(s,a)$; this property holds
pathwise even off confidence. Before the cap and before quota completion,
$\max V\ge1/(N+1)$. The entry gain hence has
$g\ge1/[B(N+1)]$, proving the rounding bound. The two tests give
$\max V/g\ge B/2$ and $\min V/g\ge B/4$, whereas monotonicity
from epoch entry gives $\max V/g\le B$.
\end{proof}

At fixed scale each gain restart reduces $g$ by a factor strictly greater
than two. The positive lower bound above limits their number to
$2+\lceil\log_2(B(N+1))\rceil$. A new scale with empty buffers executes
a transition on its first update: $V_{\rm new}\ge B-1\ge B/2$ and
both tests pass. Thus there are no infinite zero-transition restarts,
even off confidence, before a finite cap.

\subsubsection{Positive potential and absorption}
Within one gain epoch write $f_t=V_t/g$ after the backup,
$r_t=b_t(i_t)/g$, and $e_t=U^+_{i_t}-P_{i_t}f_t$. Let
$R=\sum r_t$ and $E=\sum e_t$. Let $T$ range over epoch prefixes,
including those ending at restarts.

\begin{lemma}
Intersect the scalar event with two potential-martingale events, each
of failure at most $\delta_{j,h}$. Then
\begin{align}
 \mathcal V&\le4B(R+E)+24B^2(S+L),\label{rc:eq:energy}\\
 T&\le6(R+E)+96B(S+L),\label{rc:eq:time}\\
 T&\le8R+160KBW.\label{rc:eq:absorbed}
\end{align}
\end{lemma}
\begin{proof}
The selected backup and rounding give
\[
 r_t+U^+_{i_t}-f_t(s_t)
 \ge\alpha U^+_{i_t}-\theta/g
 \ge\tfrac14-\tfrac1{256}=\tfrac{63}{256}.
\]
In particular $P f_t-f_t(s_t)\ge-r_t-e_t$.
The variance identity and $0\le f_t\le B$ imply
\[
 \Var_P(f_t)=P f_t^2-f_t(s_t)^2+
 (f_t(s_t)-P f_t)(f_t(s_t)+P f_t)
 \le P f_t^2-f_t(s_t)^2+2B(r_t+e_t).
\]
For $p=1,2$, monotonicity gives the pathwise telescope
\[
 \sum_{t<T}\{f_t(s_{t+1})^p-f_t(s_t)^p\}\le(S+1)B^p.
\]
Indeed the endpoint difference is at most $B^p$, and intermediate
changes are nonnegative decreases whose sum over all coordinates is at
most $SB^p$.

Put $M_p=\sum(f_t(s_{t+1})^p-P_{i_t}f_t^p)$.
The bounded-increment exponential supermartingale with
$\lambda=1/(16B^p)$ yields, uniformly over stopped prefixes,
\[
 -M_1\le\mathcal V/(16B)+16BL,
 \qquad -M_2\le\mathcal V/4+16B^2L.
\]
To check the second variance coefficient, use
$\Var(f_t^2)\le4B^2\Var(f_t)$.
These supermartingales follow from
$e^x\le1+x+x^2$ for $|x|\le1/16$, followed by Ville's inequality.
They are constructed before intersecting with scalar confidence.

The square telescope now gives
$3\mathcal V/4\le2B(R+E)+B^2(S+1+16L)$,
which implies \eqref{rc:eq:energy}. Summing the productive inequality gives
\[
 \frac{63T}{256}\le R+E+(S+1)B+
 \mathcal V/(16B)+16BL.
\]
Substitute \eqref{rc:eq:energy}; the reward coefficient becomes $5/4$ and
the remaining term is $B(5S/2+1+35L/2)$.
Using $S,L\ge1$ gives the conservative bound \eqref{rc:eq:time}.

By \eqref{rc:eq:window}, rounding, \eqref{rc:eq:energy}, and $W\ge2^{20}BL$,
\[
 E\le3KBW+\frac{T}{64}+\frac{25}{2048}(R+E)
       +\frac{75}{1024}B(S+L)+\frac{T}{16384}+\frac{T}{256}.
\]
For clarity, the square-root term splits into
$\frac5{256}\sqrt{T(R+E)}$ and
$\frac{5\sqrt6}{256}\sqrt{TB(S+L)}$.
Young's inequality bounds each with $T/128$ and the corresponding
$25/2048$ or $75/1024$ term. Absorbing $25E/2048$ yields
\[
 E\le T/40+R/80+4KBW+B(S+L)/12.
\]
Substitution in \eqref{rc:eq:time} gives
\[
 T\le\frac{243}{34}R+\frac{480}{17}KBW
                +\frac{1930}{17}B(S+L).
\]
Since $S\le K$ and $W\ge1+L$, the last term is at most
$1930KBW/17$. Now $243/34<8$ and $2410/17<160$, proving
\eqref{rc:eq:absorbed}.
\end{proof}

\subsubsection{Invocation bound}
Let $D_0=\max\{1,D\}$, and define
\begin{align*}
 J&=1+\lceil\log_2D_0\rceil,&
 H&=2+\lceil\log_2(4D_0(N+1))\rceil,\\
 L_{\max}&=\left\lceil\log_2
 \frac{1200K(N+1)^6(J+1)^2(H+1)^2}{\delta}\right\rceil,&
 \mathcal A_N&=2^{33}KD_0^2 H L_{\max}.
\end{align*}
There are at most $H$ gain epochs per scale. The dyadic bounds
$\sum B_j^2\le64D_0^2/3$ and $\sum B_j\le8D_0$ imply
$\sum160KBW\le\mathcal A_N$. Union over the scalar and two potential
events costs at most
$3\sum_{j,h}\delta_{j,h}=\delta\pi^4/144<\delta$.
Conditional application at each epoch entry is legitimate; conditioning
on its later success is not. Thus all invocation prefixes obey
\[
 t\le8\sum_{u<t}b_u(i_u)/g_u+\mathcal A_N
\]
with probability at least $1-\delta$.

Let $\Omega(P)$ denote the nonnegative balanced state-action flows
whose total mass is one. For inverse counts, let
$\beta_N=\max_{x\in\Omega(P)}\sum_i x_i/(n_N(i)+1)$.
Flow balance and the supersolution give
$g_u\ge g^*(b_u)\ge\beta_N$.
The rowwise harmonic sum is at most $K(1+\ln(N+1))$;
if $N\ge2\mathcal A_N$, then
\[
 \beta_N\le16K(1+\ln(N+1))/N.
\]
For integer quotas $q$, a balanced flow dominating $q$ of total mass
$C_P(q)$ gives gain at least $m/C_P(q)$ when $m$ quota units remain.
Every productive transition removes one unit, so
$\sum b_u(i_u)/g_u\le C_P(q)H_{\sum_iq_i}$.
Surplus visits have zero reward but remain charged on the left side.
These bounds control capped exploration and collection. The statistical
and stopping arguments complete the policy-identification guarantee.

\begin{lemma}[Invocation bound for the combined controller]\label{lem:mixed-invocation}
For cap $N$, put $F=1+\ln(N+1)$,
$J=1+\lceil\log_2D\rceil$ and
$H=2+\lceil\log_2(4D(N+1))\rceil$.
Let $\bar L$ be the maximum of
\[
 \left\lceil\log_2\frac{2400K(N+1)^6(J+1)^2(H+1)^2}{\eta}\right\rceil,
 \qquad
 2+\left\lceil\log_2\frac{2^{15}KS(N+1)^4}{\eta}\right\rceil.
\]
The combined controller uses the maximum specified in
Appendix~\ref{app:mixed-spec} in both its radius and $W=2^{20}BL$.
On the confidence event, its visited scale and epoch indices satisfy
$j+1\le J+1$ and $h+1\le H+1$, so that implemented logarithm is at most
$\bar L$. The second envelope dominates
$\lceil\log_2(64KS(n+1)^2/\eta)\rceil$ for every retained $n\le N$.
Thus no parameter is enlarged only inside this proof.
We claim that, uniformly over prefixes,
\begin{equation}\label{eq:mixed-invocation}
 t\le8\sum_{u<t}b_u(i_u)/g_u+a(N,\eta),\qquad
 a(N,\eta)=2^{35}DH\bar L F\sum_i\min\{d_i,D\}.
\end{equation}
The probability of this event is at least $1-\eta$.
\end{lemma}
The proof follows here.

\subsubsection{Proof of Lemma~\ref{lem:mixed-invocation}}
\begin{proof}
Both predictions dominate the current true expectation on containment.
Monotone rounded backups therefore retain the true Bellman supersolution.
The hitting-time argument of Appendix~\ref{app:rounded-controller} gives
$B\le4D$. Both reward types are nonincreasing, and before completion
$\max b\ge1/(N+1)$; hence $g\ge1/[B(N+1)]$, rounding costs at most
$1/256$, and there are at most $H$ gain epochs per scale. At a new scale
the vector is constant, so every nonempty row support and the empty-set
fallback equal $B$. The first backup is at least $B-1\ge B/2$ and executes
a transition. Thus off-event zero-transition scale loops are also finite.

Fix the analysis-only row set $C=\{i:d_i\le B\}$, let $m=K-|C|$,
$Q=F\sum_{i\in C}d_i$, and $Z=S+L$. The sparse-transport inequality
\eqref{eq:sparse-transport}, with the dyadic coefficient bound
$a_i\le2d_i/(n_i+1)$, gives
\[
 E_C\le12\sqrt{LQ\mathcal V_C}+40BLQ+SB+T_C/256.
\]
Here the implemented support uses the pre-backup vector, whereas the
variance uses the post-backup vector. Support is 1-Lipschitz in sup norm;
their nonnegative decreases sum to at most $SB$ per epoch. This lag charge
cannot be discarded. The harmonic sum uses invocation-wide counts.
Restricting the scalar-window argument to the other rows gives
\[
 E_W\le3mBW+10\sqrt{T_W\mathcal V_WL/W}+64BLT_W/W+T_W/256.
\]
Every observation on these rows is itself a selected query on these rows,
so the same variance-overlap count applies. The common potential argument
gives $\mathcal V\le4B(R+E)+24B^2Z$ and
$T\le6(R+E)+96BZ$. For explicit absorption, the coordinate square-root
terms are at most $(R+E)/64+9216BLQ+30B(LQ+Z)$.
For the scalar square root, writing $X=R+E$ and using
$\mathcal V_W\le4BX+24B^2Z$, choose the explicit Young bounds
\[
20\sqrt{TBLX/W}\le X/128+12800BLT/W,\qquad
10\sqrt{24}\,B\sqrt{TLZ/W}\le BZ/4+2400BLT/W.
\]
Including the coordinate linear terms and lag therefore gives
\[
E\le\frac3{128}(R+E)+3mBW+9286BLQ+\frac{125}{4}BZ
       +\left(\frac{15264}{2^{20}}+\frac1{256}\right)T.
\]
After division by $125/128$, the coefficients of
$(R,mBW,BLQ,BZ,T)$ are respectively
$3/125$, $384/125$, $1188608/125$, $32$, and $121/6400$.
Rounding these upward gives
\[
 E\le T/48+R/32+4mBW+10000BLQ+32BZ,
\]
and consequently
\[
 T\le8R+28mBW+70000BLQ+330BZ
       \le8R+1024mBW+2^{17}BLQ.
\]
The last step uses $BZ\le2BLQ+2mBW$, also when $C$ is empty.
This startup is at most
\[
 2^{30}LF\sum_i B\min(d_i,B).
\]
Since $\sum_B B\min(d_i,B)\le32D\min(d_i,D)$ over the dyadic scales,
summing epochs proves \eqref{eq:mixed-invocation}.
For probability accounting, the nested event costs $\eta/2$ globally.
Scalar and the two potential events each get
$\eta/[24(j+1)^2(h+1)^2]$ at scale/gain entry. Their total cost is
$\eta\pi^4/288<\eta/2$. Conditioning is only on finite entry histories;
there is no union over data-selected directions or row partitions.

\end{proof}

\subsection{Flow Balance of Adaptive Trajectories}\label{app:history-proof}
For nonnegative row demands $q$, define the circulation cost
\[
 C_P(q)=\min_{\nu\ge q}\left\{\sum_i\nu_i:
       \sum_i\nu_iP_i(s)=\sum_a\nu_{sa}\ \text{for all }s\right\}.
\]
A regenerative cycle taking a prescribed row action and then returning
to its state has expected duration at most $D+1$. Its expected visit
counts are balanced and its prescribed row count is at least one.
Adding such flows shows finiteness and the bound
$C_P(q)\le(D+1)\sum_iq_i$. The same definition proves positive
homogeneity, monotonicity and subadditivity. No circulation is supplied
to the learner in the construction below.

\begin{proposition}[Uniform-in-time bound on collection flow cost]\label{rc:prop:vertex-history}
Let $K=SA$ and
\[
 V_*={K+2S\choose S},\qquad
 L_* = \max\{1,\ln(2V_*/\eta)\}.
\]
For any nonanticipating controller, with probability at least $1-\eta$,
simultaneously for every integer $t\ge0$,
\[
 C_P(n(t))\le\frac54t+15D L_*.
\]
There is no supplied time cap, discretization accuracy, or stationarity
assumption. The logarithmic factor is
$O(S\log(e(A+2))+\log(1/\eta))$, independent of diameter and horizon.
\end{proposition}
\begin{proof}
Finite-dimensional LP duality gives
\[
 C_P(q)=\max_{b\ge0,\,h}\{q\cdot b:
                   b_i+P_i h-h(s_i)\le1\ \text{for all }i\}.
\]
We first justify that an optimizing $b$ admits a feasible bias $h$
of span at most $D$. A dual-feasible reward $b$ has optimal gain at
most one: multiply its constraint by any stationary probability flow.
The regenerative row flow also gives $b_i\le D+1$, so rewards are finite.
For completeness, a communicating finite MDP has an optimality-equation
solution for these rewards by the following compactness argument.
Discounted values have span at most $D\max_i b_i$, by following a
target-hitting policy and then the optimal discounted policy.
Center these values and take a convergent subsequence as the discount
approaches one; the Bellman equation gives
$g+h(s)=\max_a\{b_{sa}+P_{sa}h\}$.
Moreover, a normalized original dual-feasible potential $h_0\ge0$
makes $1/\alpha+h_0$ a discounted supersolution, so the limiting
$g\le1$. Nonnegativity of rewards gives $g\ge0$.
Stopping the Bellman inequality along a true target-hitting policy
gives $h(s)\ge h(z)-gD$; thus $\operatorname{span}(h)\le gD\le D$.
This replacement bias remains dual-feasible. Normalize it to $[0,D]$.

The preceding replacement yields an optimizing bias in the deterministic
feasible-bias polytope
$\mathcal F=\{h\in[0,D]^S:P_i h-h(s_i)\le1\}$.
It is compact and full dimensional: the constant vector $D/2$ satisfies
every inequality strictly. It is described by $K+2S$ inequalities.
At each vertex, $S$ linearly independent active inequalities determine
the point uniquely. Assign one such basis to each vertex; two different
vertices cannot have the same basis. Hence the number of vertices is
at most $V_*$. Every point of $\mathcal F$ is a convex combination of
these vertices. Both facts are finite-polytope properties; the vertices
depend on the true kernel only for this analysis and are never enumerated
by the learner.

Fix one vertex of $\mathcal F$, write $z=h-D\in[-D,0]^S$, and define
\[
 V_t=\sum_{u<t}\Var_{P_{i_u}}z,\quad
 M_t=\sum_{u<t}(z(s_{u+1})-P_{i_u}z),\quad
 Q_t=\sum_{u<t}(z(s_{u+1})^2-P_{i_u}z^2).
\]
The drift $d_i=P_i z-z(s_i)$ is at most one. The exact identity
\[
 \Var_{P_i}z=P_i z^2-z(s_i)^2-2z(s_i)d_i-d_i^2
             \le P_i z^2-z(s_i)^2+2D
\]
and the fixed-vector telescope imply $V_t\le D^2+2Dt-Q_t$.
Using the bounded-increment exponential argument already proved,
with $\lambda=1/(16D^2)$ and
$\Var(z(s_{u+1})^2\mid\mathcal F_u)\le4D^2\Var_{P_{i_u}}z$,
we obtain $-Q_t\le V_t/4+16D^2L_*$, uniformly over all integer $t\ge0$, outside
probability $e^{-L_*}$. Hence
\[
 V_t\le(8/3)Dt+(4/3)D^2+(64/3)D^2L_*
        \le3Dt+23D^2L_*.
\]
With $\lambda=1/(12D)$, a second event gives
\[
 M_t\le V_t/(12D)+12DL_*\le t/4+14DL_*.
\]
The all-time assertion follows by applying the nonnegative
supermartingale maximal inequality, or by increasing finite horizons
with the same threshold and taking the union of the nested crossing
events. Each crossing costs at most $e^{-L_*}$; both processes and all
vertices together cost at most $\eta$.

For each fixed realized path and time, $M_t(h)$ is linear in $h$.
Therefore its upper bound at every vertex extends by convex combination
to every $h\in\mathcal F$. Selecting the optimizing bias
after seeing counts is now valid on this one simultaneous event.
Duality and the endpoint telescope give
\[
 C_P(n(t))\le t+h(s_0)-h(s_t)+M_t(h)
             \le\tfrac54t+D+14DL_*
             \le\tfrac54t+15DL_*.
\]
There is no grid approximation term. Finally
$\binom{K+2S}{S}\le[e(K+2S)/S]^S=[e(A+2)]^S$.
\end{proof}

\subsection{Pilot Quality and Quota Collection}
\paragraph{Pilot allocation.}
Run inverse counts for $n$ steps and freeze
$w_i=(c_i+1)/(n+K)$. At the endpoint let
$\beta=\max_{x\in\Omega(P)}\sum_i x_i/(c_i+1)$.
Every preceding optimistic gain dominates $\beta$.
If $n\ge2a(n,\eta)$ and $n\ge K$, the harmonic sum gives
\[
 G_P(w)=(n+K)\beta\le32K(1+\ln(n+1)),\qquad
 1/w_{\min}\le n+K.
\]
Let $\mathcal C_P(v)$ be minimum balanced-flow mass dominating $v$.
It is homogeneous and subadditive. Execute a row once and return to its
source to obtain $\mathcal C_P(\mathbf1)\le K(D+1)$.
Proposition~\ref{rc:prop:vertex-history} gives
\[
 \mathcal C_P(w)\le
 \frac{5n/4+15DL_*+K(D+1)}{n+K}\le\frac32
\]
when $n\ge60DL_*+4K(D+1)$, where
$L_* =\max\{1,\ln[2\binom{K+2S}{S}/\eta]\}$.
As $A_{\rm mix}\ge KD$, all pilot conditions hold at
$n=\widetilde O(A_{\rm mix})$. None is tested by the learner.

\paragraph{Statistical demand and actual collection.}
For pilot-fixed $M\ge2/w_{\min}$ and $q_i=\lceil Mw_i\rceil$,
\[
 \mathcal C_P(q)\le3M/2+K(D+1).
\]
If $m$ quota units remain, a balanced flow dominating $q$ gives unfinished
reward gain at least $m/\mathcal C_P(q)$. Summing productive visits
therefore gives $R\le\mathcal C_P(q)H_{\sum_iq_i}$.
By \eqref{eq:mixed-invocation}, any cap exceeding
\[
 8[3M/2+K(D+1)](1+\ln(M+K))+a(\mathrm{cap},\eta)
\]
forces completion. Otherwise the prefix inequality at the cap is a
contradiction. Surplus interactions and rounding are both charged.

The statistical lemma without pilot-based projection in Appendix~\ref{app:span-adaptive} is
sound conditional on any finite pilot and positive frozen weights.
Its sufficient mass, using auxiliary bias span $C_0\le DG_P(w)$, is
\[
 \widetilde O\!\left(
 G_P(w)(B_*+1)/\epsilon^2+DG_P(w)/\epsilon+1/w_{\min}\right).
\]
Thus a good pilot has sufficient mass
$\widetilde O(n+K(B_*+1)/\epsilon^2+KD/\epsilon)$.
Use exactly its discount grid, common confidence parameter and corrected
numerical endpoints, with complete rounded fallback. Virtual fixed row
prefixes are defined before intersecting collection success.

\subsection{Quantile concentration}\label{app:quantile-concentration}
This calculation is used with the finite source-absorption grids in
Appendices~\ref{app:auxiliary-coverage} and~\ref{app:span-adaptive}.
Condition on a finite pilot history. An equivalent construction assigns
independent infinite successor streams to the rows after that history,
revealing the next unused entry when a row is visited. The Markov property
gives the original interaction law. Frozen weights and prescribed mass
fix every row-prefix length; unused entries define virtual prefixes even
if quotas are never completed. Thus conditioning here is on the pilot.
For each row, consider $n$ independent fresh successors and a finite collection of random vectors independent of
that row's samples. Let $R$ be the number of row--vector pairs, including
all coordinate-level truncations, and let $\ell\ge1$. After a union,
the bounds below hold on this collection and transfer to fitted targets
within sup distance $1/N$. Set $A_0=16\ell$ and
$\beta=A_0/\max\{n-1,1\}$.
They imply one-sided validity of the quantile functionals without pilot-based projection and,
for each covered fitted target $h$ of span $H$,
\[
 0\le Ph-F^-(h)\le
 16\sqrt{\frac{A_0\operatorname{Var}_P(h)}n}
       +\frac{8A_0H}n+\frac{32}N,
\]
with the reflected statement for $F^+$ at its covered targets.
The grid arguments supply the
required independence and approximation; no successful-collection event
is used for conditioning.

For one independent vector $X$, write $B=\spn(X)$, $v=\Var_P(X)$,
$\widehat v=\Var_{\widehat p}(X)$ and $b=\ell/n$. Scalar Bernstein for
$X$ and $(X-PX)^2$ gives, with failure at most $4e^{-\ell}$,
\[
 |(\widehat p-P)X|\le\sqrt{2vb}+\tfrac23Bb,\qquad
 |\widehat p(X-PX)^2-v|\le\sqrt{2B^2vb}+\tfrac23B^2b.
\]
The second application uses variance at most $B^2v$ and range at most $B^2$.
Subtracting the squared empirical mean error yields
$\widehat v\le2v+2B^2b$. For $b\le1/16$, the inequalities
$\sqrt{2B^2vb}\le v/4+2B^2b$ and
$((\widehat p-P)X)^2\le4vb+(8/9)B^2b^2$ give
$\widehat v\ge v/2-3B^2b$. Hence
$|(\widehat p-P)X|\le2\sqrt{b\widehat v}+5Bb$.
A union over $R$ row--vector pairs costs at most $4Re^{-\ell}$.

For $b\le1/16$, standard deviations change by at most sup distance, means by sup distance
and spans by twice that distance. Transfer to the covered targets within
$1/N$, including their clipped transforms, to obtain
\begin{align*}
 |(\widehat p-P)h|&\le2\sqrt{b\Var_P(h)}+b\spn(h)+3/N,\\
 \sqrt{\Var_{\widehat p}(h)}&\le2\sqrt{\Var_P(h)}+2\spn(h)\sqrt b+4/N,\\
 |(\widehat p-P)h|&\le2\sqrt{b\Var_{\widehat p}(h)}+5b\spn(h)+4/N.
\end{align*}
If $\beta\le1$, then $n\ge A_0+1$, $b\le1/16$ and $\beta\ge16b$.
Half the sum of the two penalty branches dominates
$2\sqrt{b\Var_{\widehat p}(z)}+5b\spn(z)$, and their maximum is at
least that half-sum. The last concentration bound and the $5/N$ subtraction
give $F^-(h)\le Ph$. The minimum floor cannot violate
this inequality. The case $\beta>1$ is immediate; reflection gives the
upper bound.

\paragraph{Quantile clipping and row error.}
For $\beta\le1/2$, let $c$ be the threshold and $J_>=\{h>c\}$.
Maximality gives $\widehat p(J_>)<\beta$. Put
$d=\widehat p(h-\min(h,c))$. If $c\ge\widehat p h$, then
$d\le\sqrt{\beta\widehat v}$ by Cauchy--Schwarz. Otherwise
$\widehat p h-c\le d$, so $d\le\sqrt{\beta\widehat v}+\beta d$.
Thus $d\le2\sqrt{\beta\widehat v}$. Clipping cannot increase variance,
by the independent-copy squared-difference formula, and therefore
\[
 0\le\widehat p h-F^-(h)
 \le3\sqrt{\beta\widehat v}+\beta\spn(h)+5/N.
\]
Substitute the preceding concentration bounds, $b\le\beta/16$ and
$\spn(h)\le H$ to get
$Ph-F^-(h)\le7\sqrt{\beta\Var_P(h)}+3\beta H+20/N$.
Since $\beta=A_0/(n-1)\le2A_0/n$, this implies the displayed row-error bound above.
For $\beta>1/2$ use the trivial range bound $H\le4A_0H/n$ for $n\ge2$;
the $n=1$ case is immediate. Reflection finishes the upper sign.

\subsection{Nonuniform estimation}
\label{app:auxiliary-coverage}

This section establishes the fixed-mass statistical guarantee used in
the stopping proof: quality at a prescribed post-pilot sample mass.
Appendix~\ref{app:span-adaptive} turns it into an observable stopping rule.
The planner is the pessimistic quantile operator without pilot-based projection of
\citet{zzc}, with conservative confidence and numerical constants.
In particular, \citet[Theorem~4]{zc} already obtain the uniform-generative
oracle rate below using span-constrained planning and parameter selection.

\subsubsection{Algorithm and theorem}
Let $P$ be communicating, $r\in[0,1]^{SA}$ be known, and let positive
rational weights $w$, summing to one, and an integer $M\ge1$ be fixed
before the fresh observations. Set $n_i=\lceil Mw_i\rceil$,
$N=\sum_i n_i$, $a=(M+1)^{-1}$, $\gamma=1-a$, and
\begin{align*}
 L&=\lceil\log_2(4M)\rceil,\qquad J=\lceil128N/a^2\rceil,\\
 \alpha&=16\left\lceil\log_2
 \frac{32SA(S+1)((J+1)^2+1)(L+1)}{\eta}\right\rceil.
\end{align*}
If $M<2/w_{\min}$, return any policy and lower bound zero. Otherwise,
use the chronological first $n_i$ fresh samples from each row. Let $F_i$
be the lower quantile functional of Appendix~\ref{app:nonuniform-statistic}, with penalty parameter $\alpha$ and
additive correction $5/N$. Compute the fixed point of
\[
 (TV)_s=\max_b\{r_{sb}+\gamma F_{sb}(V)\}
\]
and return its discounted-greedy policy. Numerical acceptance is specified
below. No true kernel, diameter, bias span, or auxiliary potential is an input.

For analysis only, put $t_i=\alpha/n_i$. Consider any scalar $\kappa$ and
potential $z\in\mathbb R^S$ satisfying the familiar Bellman supersolution
inequalities for the auxiliary reward $t$:
\begin{equation}\label{eq:auxiliary-witness}
 t_i+P_i z\le\kappa+z(s(i))\quad\text{for every row }i.
\end{equation}
Write $c_z=\operatorname{span}(z)$.

\begin{theorem}[Fixed-prefix oracle bound]\label{thm:auxiliary-oracle}
With probability at least $1-\eta$, simultaneously for every deterministic
stationary comparator $\pi$ with constant gain $g^\pi$ and a finite
Poisson bias of span $B_\pi$, every witness
\eqref{eq:auxiliary-witness}, and every initial state $s$, the exact output obeys
\begin{equation}\label{eq:auxiliary-oracle}
 g^{\widehat\pi}(s)\ge g^\pi
       -512\sqrt{(B_\pi+1)\kappa}-32768c_z.
\end{equation}
For the numerically certified output, subtract $4\nu+q$ on the right,
where $\nu$ bounds the sup-norm fixed-point error and $q$ is the largest
outward row-enclosure width. The same event permits taking the best
comparator and the best witness; neither is computed by the learner.
\end{theorem}

Two witness choices clarify the result. First,
$\kappa=\max_i t_i,z=0$ is always feasible. Equal row counts $n_i=n$
therefore give
\[
 g^{\widehat\pi}(s)\ge
 \sup_\pi\left\{g^\pi-512\sqrt{(B_\pi+1)\alpha/n}\right\},
\]
up to numerical error, with no diameter or policy-hitting-radius condition.
Second, let $A_t=\max_{x\in\Omega(P)}x^\top t$. An optimal Bellman bias
for reward $t$ supplies a witness with $\kappa=A_t$ and $c_z\le DA_t$.
Indeed, $t\ge0$ and \eqref{eq:auxiliary-witness}, stopped along a controlled
path from $x$ to $y$, imply $z(y)-z(x)\le A_t\mathbb E_x\tau_y$.
Bounded $z$ and integrable hitting justify stopping first at $\tau_y\wedge n$
and taking limits. Optimizing navigation and varying $x,y$ gives the bound.
Thus the theorem also recovers the diameter-dependent transport guarantee,
but can be strictly sharper. The supersolution inequality itself is a
standard average-reward tool; the new step is its use with the retained
negative clipping term below.

\subsubsection{Statistical event and virtual clipping}
For $H_j=2^j$, $j=0,\ldots,L$, define
$V_j=C_{H_j}TV_j$, where $C_Hx=\min\{x,\min(x)+H\}$ coordinatewise.
Quantile monotonicity, shift equivariance and range preservation
\citep[Lemma~B.2 and Appendix~E.1]{zzc} make these maps monotone
$\gamma$-contractions on $[0,1/a]^S$. Consequently
$V_j\le V_\infty$, where $TV_\infty=V_\infty$; the largest cap is at least
$M+1$, so its fixed point equals $V_\infty$. Common $\alpha,\gamma$ across
the virtual caps are essential. Only $V_\infty$ is actually computed.

Here global clipping requires a two-parameter decoupling. For a source
state $s$, replace its output by a deterministic $u$, and every other
output by $\min\{(TX)_j,l\}$, with $0\le u\le l\le1/a$. The resulting
fixed point is independent of all samples originating at $s$. Choosing
$u=V_j(s)$ and $l=\min\{\min(V_j)+H_j,1/a\}$ recovers $V_j$.
Round both parameters upward on $\{d/(aJ):0\le d\le J\}$.
Ordering is preserved and contraction bounds the fixed-point error by
$1/(a^2J)\le1/(128N)$. Include the vector and its $S$ coordinate-level
truncations, whose approximation error is at most $1/(64N)$.

There are at most $SA(S+1)(J+1)^2(L+1)$ row--vector pairs.
Conditional scalar Bernstein bounds for a vector and its squared centered
value, as derived in Appendix~\ref{app:quantile-concentration}, and a union bound
with $\ell=\alpha/16$ have failure at most $\eta$. In particular, with
$x=\ell/n\le1/16$, covered fitted vectors satisfy
\begin{align*}
 |(\widehat P-P)h|&\le2\sqrt{x\mathbb V_P h}+x\operatorname{span}(h)+3/N,\\
 \sqrt{\mathbb V_{\widehat P}h}&\le2\sqrt{\mathbb V_Ph}
                   +2\operatorname{span}(h)\sqrt{x}+4/N,\\
 |(\widehat P-P)h|&\le2\sqrt{x\mathbb V_{\widehat P}h}
                   +5x\operatorname{span}(h)+4/N.
\end{align*}
These apply also to the covered truncation. If
$\beta=\alpha/(n-1)\le1$, then $\beta\ge16x$, and the maximum of the
two penalty branches dominates $2\sqrt{x\mathbb V_{\widehat P}h}
+5x\operatorname{span}(h)$. The $5/N$ correction covers the remainder.
If $\beta>1$, the operator returns the minimum directly. Its quantile
truncation loses at most $2\sqrt{\beta\mathbb V_{\widehat P}h}$ when
$\beta\le1/2$; the larger-$\beta$ case follows from range preservation.
The calculation in Appendix~\ref{app:quantile-concentration} therefore yields,
at every exact virtual fixed point and with $h=V_j-\min V_j$,
\begin{equation}\label{eq:auxiliary-row}
 0\le e_i:=P_ih-F_i(h)
 \le16\sqrt{t_i\mathbb V_{P_i}h}+8H_jt_i+32/N.
\end{equation}

\subsubsection{Stopped comparison}
Fix a comparator of bias span $B$. Its true discounted value $U$ satisfies
$\operatorname{span}(U)\le2B$ and $|aU_s-g^\pi|\le aB$ by applying
the linear discounted Bellman operator to $g^\pi/a+h^\pi$.
Choose a virtual $H\ge\max\{1,4B\}$, and put
$V=C_HTV$, $h=V-\min V$, $g=a\min V$, and
$\Delta=H-\operatorname{span}(U)\ge H/2$.
Start $\pi$ at a minimizer $s_0$ of $h$, stopping at the first visit
$\tau$ to $\{h=H\}$. Infinite $\tau$ is allowed, with $\gamma^\infty=0$.
Define
\[
 b=\mathbb E\gamma^\tau,\quad
 d_i=\frac{a}{1-b}\mathbb E\sum_{t<\tau}\gamma^t
             \mathbf1\{(s_t,\pi(s_t))=i\},\quad
 \lambda=\frac{a}{\gamma(1-b)}.
\]
Then $d$ is a probability occupation, $b\le\gamma<1$, and
$dP-d_S=\lambda(\xi-\delta_{s_0})$, where
$\xi=(1-b)d_S+\nu^{\rm stop}$ and
$\nu^{\rm stop}_s=\mathbb E[\gamma^\tau\mathbf1\{s_\tau=s\}]$.
Multiplying \eqref{eq:auxiliary-witness} by $d$ gives
\begin{equation}\label{eq:auxiliary-transport}
 d^\top t\le\kappa+\lambda(z(s_0)-\xi^\top z)
             \le\kappa+\lambda c_z.
\end{equation}

Off the cap, $W=U-V$ satisfies $W\le\gamma P_\pi W+\gamma e$;
on it, $W\le W(s_0)-\Delta$. Nonnegativity of $W$ is not needed,
including for a suboptimal comparator. Discounted stopping and bounded
terminal remainder yield
\begin{equation}\label{eq:auxiliary-stop}
 g^\pi-g\le aB+\gamma E_d-\frac{ab\Delta}{1-b},\qquad E_d=d^\top e.
\end{equation}
Assume $aH\le1$. Off the cap,
$P_\pi h-h\le(g+aH)/\gamma+e$. The flow identities applied to $h,h^2$
and $h(s_0)=0$, $h=H$ on the cap imply
\[
 V_d:=d^\top\mathbb V_Ph
 \le 2H(g+aH)/\gamma+2HE_d\le8H+2HE_d.
\]
For clarity, writing $v=P_\pi h-h$, expand the variance as
$Ph^2-h^2-2hv-v^2$, use $-2hv\le2Hv_-$ and
$v_-=v_+-v$. The remaining flow terms are
$\tfrac a\gamma d^\top(h^2-2Hh)-\lambda bH^2\le0$.
This proves the displayed bound without treating $d$ as stationary.

Cauchy--Schwarz in \eqref{eq:auxiliary-row} and absorption of
$16\sqrt{2H(d^\top t)E_d}\le E_d/2+256H(d^\top t)$ give
\[
 E_d\le u\sqrt{Hd^\top t}+vHd^\top t+2\rho,
 \qquad u=64\sqrt2,\quad v=528,\quad\rho=32/N.
\]
Substitute \eqref{eq:auxiliary-transport} into
\eqref{eq:auxiliary-stop}, retain its negative term, and use
$ab/(1-b)=\gamma\lambda-a$. If $vc_z\le1/4$, the terms depending
on $\lambda$ are bounded by
\[
 u\sqrt{Hc_z\lambda}-\Delta\lambda/2
 \le\frac{u^2Hc_z}{2\Delta}\le8192c_z.
\]
It follows that
\begin{equation}\label{eq:auxiliary-intermediate}
 g^\pi-g\le\tfrac54aH+u\sqrt{H\kappa}
                    +vH\kappa+8192c_z+64/N.
\end{equation}

\subsubsection{Policy guarantee}
A stationary average of \eqref{eq:auxiliary-witness} shows
$\kappa\ge A_t\ge\alpha/N$; pointwise it gives $t_i\le\kappa+c_z$.
In the early branch, some $n_i\le2$, so $\kappa+c_z\ge\alpha/2$.
Either $32768c_z\ge1$ or $512\sqrt\kappa>1$, making the theorem trivial.
Otherwise $M\ge2/w_{\min}\ge2SA$, hence $N\le3M/2$ and
$1/M\le3\kappa/(2\alpha)\le3\kappa/2$.

The claim is again trivial if $512\sqrt{(B+1)\kappa}\ge1$ or
$32768c_z\ge1$. In the remaining regime, $vc_z<1/4$ and
$M\ge8(B+1)$: otherwise $(B+1)\kappa>\alpha/12$ contradicts that regime.
Choose the first dyadic $H\ge\max\{1,4B\}$, so $H\le8(B+1)\le M$.
With $x=\sqrt{(B+1)\kappa}<1/512$,
\eqref{eq:auxiliary-intermediate} becomes
\[
 g^\pi-g\le256x+4303x^2+8192c_z
             \le512x+32768c_z.
\]
Since $V_H\le V_\infty$, its minimum is inherited by the computed fixed
point. At $V_\infty$, the pessimistic discounted-greedy action and
\eqref{eq:auxiliary-row} yield
$r_{\widehat\pi}+P_{\widehat\pi}h-h\ge a\min V_\infty$.
Telescoping bounded $h$ proves the gain statement from every initial state.
The argument holds for each comparator and witness without a union bound.

For finite computation, let $W$ have certified distance at most $\nu$ from
$V_\infty$, and let lower row enclosures have width at most $\zeta$.
Choose the action maximizing $r+\gamma F_{\rm lo}(W)$.
Its discounted residual is at least
$TV_\infty-V_\infty-(1+\gamma)\nu-\gamma\zeta$.
The enclosure width and the exact functional's minimum floor give
\[
 F_{\rm lo}(W)\ge F(W)-\zeta\ge\min W-\zeta
 \ge\min V_\infty-\nu-\zeta.
\]
Adding $aF_{\rm lo}(W)$ to the discounted residual and using
$a+\gamma=1$ therefore yields
\[
\begin{aligned}
 r_{\widehat\pi}+F_{\rm lo}(W)-W
 &=r_{\widehat\pi}+\gamma F_{\rm lo}(W)-W+aF_{\rm lo}(W)\\
 &\ge-(1+\gamma)\nu-\gamma\zeta
      +a(\min V_\infty-\nu-\zeta)\\
 &=a\min V_\infty-2\nu-\zeta.
\end{aligned}
\]
One-sidedness at the exact fixed point and Lipschitz continuity give
$F_{\rm lo}(W)\le PW+2\nu$. Thus subtracting $2\nu$ from the computed
residual is safe, and its minimum satisfies
\[
 L_{\rm raw}\ge a\min V_\infty-4\nu-\zeta.
\]
The resulting numerical loss is at most $4\nu+\zeta$.

For a rational normalized proposal $h$, enclose $Th-h$ coordinatewise
by $[l_s,u_s]$, put $c=(\min l+\max u)/2$ and
$R=(\max u-\min l)/2$. The shifted vector $W=h+c/a$ has residual
at most $R$, so $\nu=R/a$ is certified. A complete fallback iterates $T$
with downward dyadic rounding of mesh $d\le a^2\nu/128$ and row error
at most $d$. After
$\lceil\lceil\log_2(32/(a^2\nu))\rceil/a\rceil$ iterations its distance
is at most $3a\nu/64$, and the final residual check passes.
Re-rounding bounds intermediate bit lengths, making this fallback
polynomial in the \emph{raw} data size, $S,A$, reward bit length and
$\log(1/\nu)$.

\subsection{Stopping with unknown bias span}
\label{app:span-adaptive}
This section proves when the nonuniform statistical test becomes narrow
enough to stop. It uses the fixed-mass result of
Appendix~\ref{app:auxiliary-coverage}, independently of how the pilot was obtained. Let $B_*=\spn(h^*)$ for an
optimal \emph{Bellman} bias, satisfying
$g^*+h^*_s=\max_a\{r_{sa}+P_{sa}h^*\}$. The discount grid, row parameter and outward
intervals are specified in Appendix~\ref{app:nonuniform-statistic}.

\subsubsection{Stopping lemma}
Fix any finite pilot history, its positive frozen weights $w$, and a
prescribed batch $M\ge2/w_{\min}$. Write $N=\sum_i\lceil Mw_i\rceil$,
$t_i=\alpha/n_i$, and consider an analysis-only witness
\[
 t_i+P_i z\le\kappa+z_s,\qquad C=\spn(z).
\]
Every stationary occupation gives $\kappa\ge\alpha/N$.
Let $b=B_*+1$. We show that the observable rule stops at this batch if
\begin{equation}\label{eq:span-sufficient}
 \kappa\le\frac{\epsilon^2}{2^{30}b},\qquad
 C\le\frac{\epsilon}{2^{26}}.
\end{equation}

The source-absorption construction of Appendix~\ref{app:auxiliary-coverage}
applies to both quantile signs and the finite discount grid. For the upper
operator, replace a source output by $u$ and cap other outputs at $l$;
at its unclipped fixed point take $u=V_s,l=1/(1-\gamma)$.
Both parameters lie on the same bounded rectangle as before. Rounding
upward incurs at most $1/((1-\gamma)^2J)\le1/(128N)$ error, since
$J\ge128N/a_{\min}^2$. Include the negative vectors and their
coordinate-level truncations for the reflected sign. The factor two for
signs and the additional discount count are covered by the displayed
$64(L_g+2)$ factor in $\alpha$.
At every exact upper fixed point, $h=V-\min V$, $H=\spn(V)$, the event gives
\begin{equation}\label{eq:upper-row-span}
 0\le e_i:=F_i^+(h)-P_i h
 \le16\sqrt{t_i\mathbb V_{P_i}h}+8Ht_i+\rho,
 \qquad\rho=32/N.
\end{equation}
Here $H$ is the \emph{actual} fitted span: transferring scalar Bernstein
bounds from the grid uses the target vector's own span plus the explicit
rounding slack. The lower virtual
points simultaneously satisfy Theorem~\ref{thm:auxiliary-oracle} with the
same larger $\alpha$.

\subsubsection{Optimistic value span}
Fix a grid discount $\gamma\ge1/2$, put $a=1-\gamma$, and let $V$ be the
optimistic fixed point. Its invariant cube is $[0,1/a]^S$, so $aH\le1$.
Let $\pi_+$ be discounted greedy and $d$ its normalized infinite discounted
occupation under the \emph{true} kernel from an arbitrary initial state $s$.
The flow identity gives
\[
 dP-d_S=\frac a\gamma(d_S-\delta_s),\qquad
 d^\top t\le\kappa+\frac a\gamma C=:\kappa'.
\]
Write $g=a\min V$ and $E=d^\top e$. On greedy rows,
$P_{\pi_+}h-h=(g-r+ah)/\gamma-e\ge-1/\gamma-e$.
Expanding conditional variance and using the flow identity,
\begin{align*}
 d^\top\mathbb V_Ph
 &=d^\top(Ph^2-h^2)-2d^\top[h(Ph-h)]-d^\top(Ph-h)^2\\
 &\le (a/\gamma)H^2+2H(1/\gamma+E)\le6H+2HE.
\end{align*}
Combining this with \eqref{eq:upper-row-span} yields
$E\le16\sqrt{\kappa'(6H+2HE)}+8H\kappa'+\rho$.
Use $\sqrt{x+y}\le\sqrt x+\sqrt y$ and
$16\sqrt{2H\kappa'E}\le E/2+256H\kappa'$ to obtain
\begin{equation}\label{eq:upper-energy-bound}
 E\le u\sqrt{H\kappa'}+vH\kappa'+2\rho,
 \qquad u=64\sqrt2,\quad v=528.
\end{equation}
The constants also bound the sharper coefficient $32\sqrt6$.

Let $U_\gamma$ be the true discounted optimal value. Since
$T_PV\le V$, monotone contraction gives $V\ge U_\gamma$.
Greediness and optimality respectively give
$V-U_\gamma\le\gamma P_{\pi_+}(V-U_\gamma)+\gamma e$.
Iterating this bounded inequality and applying \eqref{eq:upper-energy-bound}
shows $\max(V-U_\gamma)\le(\gamma/a)\sup_s E_s$.
The optimal Bellman equation, applied to $h^*+g^*/a$, bounds
$\spn(U_\gamma)\le2B_*$. Nonnegativity of $V-U_\gamma$ then implies
\begin{equation}\label{eq:upper-bootstrap}
 H\le2B_*+\sqrt{u^2\gamma T/a}\sqrt H+vTH+2\gamma\rho/a,
 \qquad T=\gamma\kappa/a+C.
\end{equation}
No small-span assumption was made to derive this inequality.

Since $N\le3M/2$, conditions~\eqref{eq:span-sufficient} imply
$M\ge2^{29}b/\epsilon^2$. The dyadic grid therefore contains
\[
 \frac{\epsilon}{128b}\le a\le\frac{\epsilon}{64b}.
\]
For this point,
$T\le9\epsilon/2^{26}$,
$u^2\gamma T/a\le9b/64$, $vT<1/1000$, and
$2\gamma\rho/a\le b/131072$.
If $H>4b$, dividing \eqref{eq:upper-bootstrap} by $H$ bounds its right
side by less than $1/2+3/16+1/1000+1/(4\cdot131072)<1$, a contradiction.
Consequently $H\le4b$.

For any true recurrent-class occupation $x$ of $\pi_+$, stationary flow
gives $x^\top t\le\kappa$ and
$x^\top\mathbb V_Ph\le4H+2Hx^\top e$.
The same absorption gives
$x^\top e\le u\sqrt{H\kappa}+vH\kappa+2\rho$.
Averaging the discounted equation yields
$g= x^\top r-a x^\top h+\gamma x^\top e\le g^*+x^\top e$.
Thus the exact upper endpoint obeys
\[
 a\max V-g^*\le aH+u\sqrt{H\kappa}+vH\kappa+2\rho
 \le\epsilon\left(\frac1{16}+\frac{2u}{2^{15}}
                         +\frac{2112+64}{2^{30}}\right)<\epsilon/8.
\]
The lower theorem, using a Bellman-greedy optimal comparator, gives
$g^*-L\le512\sqrt{b\kappa}+32768C
\le\epsilon/64+\epsilon/2048<\epsilon/32$.
The lower numerical endpoint loses at most $5\nu$, and an accepted upper
endpoint adds at most $2a\nu\le\nu$, where $\nu=\epsilon/128$.
Their total width is therefore less than $\epsilon$.
The complete rounded fallback in Appendix~\ref{app:auxiliary-coverage}
also applies to $F^+$: reflection preserves monotonicity, constant-shift
equivariance and the invariant cube.

\subsubsection{Conditional guarantees}
Let $G=\max_{x\in\Omega(P)}\sum_i x_i/w_i$. An optimal auxiliary Bellman
bias $z_0$ for reward $1/w_i$ has span $C_0\le DG$, by the nonnegative-reward
stopped hitting argument already proved in Appendix~\ref{app:auxiliary-coverage}.
Then $\kappa=\alpha G/M,z=(\alpha/M)z_0$ satisfy the witness inequalities.
The preceding lemma ensures stopping by a dyadic mass of order
\[
 \widetilde O\!\left(\frac{G(B_*+1)}{\epsilon^2}
                       +\frac{C_0}{\epsilon}+\frac1{w_{\min}}\right).
\]

For the empirical pilot in the main proof, $G=\widetilde O(K)$,
$1/w_{\min}=\widetilde O(A_{\rm mix})$ and $C_0\le DG$.
Appendix~\ref{app:rowwise-mixed} charges the resulting mass using its
combined quota collector and sums the failure allowances over attempts.
The statistical event here is conditional only on the observed finite
pilot. It is first proved for virtual row prefixes and then intersected
with actual collection. Conditioning uses only the finite pilot history.

\paragraph{Attribution.}
Span-agnostic generative stopping is due to \citet[Theorem~2]{zc}.
The additional calculation here is the nonuniform optimistic span bound
\eqref{eq:upper-bootstrap}, controlled by an auxiliary Bellman potential.
The interaction improvement additionally requires the affordable pilot
and actual-collection argument of the main proof.

\subsection{Proof of the Main Theorem}\label{app:rowwise-mixed}
Write $d_i=|\operatorname{supp}P_i|$, $K=SA$, and
$A_{\rm mix}=D\sum_i\min\{d_i,D\}$. Support and diameter remain unknown.
The following refinement combines the nested support interface with the
rounded scalar-window controller, and then uses the statistical
certificate of Appendix~\ref{app:nonuniform-statistic}.

We now combine the pilot, collection, and statistical guarantees to prove
Theorem~\ref{thm:main-total}. The combined invocation bound uses the nested
row transport inequality~\eqref{eq:sparse-transport} and the scalar-window
bounds. Proposition~\ref{rc:prop:vertex-history} then controls the flow cost
of the empirical pilot allocation. Quota collection converts its
nonuniform sample requirements into interactions, and the width test
certifies the output policy. With the data boundaries in
Appendix~\ref{app:mixed-spec}, it remains to show that enumeration reaches
a sufficient attempt and to combine its conditional confidence events.

\paragraph{Total interaction cost.}
Use the outer enumeration of Appendix~\ref{app:outer-schedule}, with $H_j=2^j$ and four event allowances $\eta_j=\delta/[64(j+1)^4]$ per attempt. Each attempt costs at most $2H_j$.
Conditional union and integration over finite histories give total failure
$\sum_j4(j+1)^2\eta_j=\delta\pi^2/96<\delta$.

\paragraph{A common sufficient attempt.}
Write $Y=A_{\rm mix}+K(B_*+1)/\epsilon^2+KD/\epsilon+1$ and set
\[
 \ell=1+\left\lceil\log_2
 \frac{2^{16}KS(D+1)(B_*+2)}{\delta\epsilon}\right\rceil,
 \qquad \ell_J=\ell+J.
\]
These are analysis quantities. For cap $N\le2^J$ and allowance $\eta_J$,
the explicit logarithms of Lemma~\ref{lem:mixed-invocation} satisfy
$H\le4\ell_J$, $\bar L\le32\ell_J$ and $F\le2\ell_J$.
In particular, $a(N,\eta_J)\le2^{50}A_{\rm mix}\ell_J^3$.
Choose the deterministic dyadic pilot length
\[
 n_J=2^{p_J}:=2^{\lceil\log_2(2^{52}A_{\rm mix}\ell_J^3)\rceil}.
\]
Whenever $n_J\le2^J$, it meets $n_J\ge2a(n_J,\eta_J)$,
$n_J\ge K$, and $n_J\ge60DL_*+4K(D+1)$ with allowance $\eta_J$.
Consequently every pilot history on the two pilot events satisfies
\[
 G_P(w)\le\overline G_J:=64K\ell_J,\qquad
 w_{\min}^{-1}\le n_J+K,\qquad \mathcal C_P(w)\le3/2.
\]
These common bounds remove the pilot randomness from the sufficient
mass. For $2/w_{\min}\le M\le2^J$, the statistical total is at most
$3M/2$. Its grid size is at most $3072M^3+1$, so the definition of
$\alpha$ in Appendix~\ref{app:nonuniform-statistic}, including
$\eta_J=\delta/[64(J+1)^4]$, gives $\alpha\le\overline\alpha_J:=512\ell_J$.
Take $M_J=2^{m_J}$ to be the smallest dyadic number at least
\[
 \max\left\{2(n_J+K),\quad
 \frac{2^{30}\overline\alpha_J\overline G_J(B_*+1)}{\epsilon^2},\quad
 \frac{2^{26}\overline\alpha_JD\overline G_J}{\epsilon}\right\}.
\]
Thus $M_J\ge2/w_{\min}$, and both inequalities in
\eqref{eq:span-sufficient} hold uniformly over those histories, using
$C_0\le DG_P(w)$.

The dyadic rounding gives $n_J<2^{53}Y\ell_J^3$ and
$M_J<2^{56}Y\ell_J^3$. Hence it suffices to choose
\[
 2^J>2^{64}Y\ell_J^4.
\]
This ensures $n_J,M_J\le2^J$, so $p_J,m_J\le J$ and the logarithmic
bounds just used apply. It also makes the collector cap exceed
\[
 8[3M_J/2+K(D+1)](1+\ln(M_J+K))+a(2^J,\eta_J),
\]
which is at most $2^{62}Y\ell_J^4$. Quotas must therefore finish,
and the corrected statistical test accepts on its event.
For an explicit sufficient stage, put
\[
 v=\lceil\log_2(Y+2)\rceil,\qquad u=v+\ell+2,\qquad
 J=v+8\lceil\log_2u\rceil+128.
\]
Then $J+\ell\le137u$ and $2^{128}>2^{64}137^4$, proving the required
strict inequality. Also $2^J=\widetilde O(Y)$. If the learner has not
stopped earlier, the fixed pair $(p_J,m_J)$ occurs at stage $J$ and
accepts on the joint event. Summing every preceding attempt costs less
than $4(J+1)^22^J=\widetilde O(Y)$, as claimed.

The computational assertion follows from Appendix~\ref{app:computation}.

\section{Proof of the Lower Bound}\label{app:matching-summary}
Theorem~\ref{thm:lower} uses one known reward table across a family of
unknown kernels. Fix $S\ge12$, $A\ge3$, $0<\delta\le1/16$,
$0<\epsilon\le1/32$, and
$\overline D\ge\max\{32,8(\lceil\log_{A-1}\lfloor S/3\rfloor\rceil+1)\}$.
The class contains communicating models with at most $S$ states, at most
$A$ actions per state, diameter at most $\overline D$, and total support
at most $J_0$, with $2SA\le J_0\le S^2A$.
The learner must satisfy timely stopping and all-start-state correctness
jointly on every instance, with deterministic stationary output. For a
common deterministic stopping budget, the lower bound is
$\Omega(SA\overline D\epsilon^{-2}\log(1/\delta))$, including with
generative access. The proof immediately following supplies the constants.

\subsection{Bounded-time lower bound}
\label{app:lower-proof}
The following construction concerns a class of fixed known rewards: the
reward table is the same for every transition kernel in the hard family.
It does not assert a lower bound for every reward table. In particular,
constant rewards require no policy identification.

\begin{theorem}[Bounded-time lower bound]\label{thm:lower}
Let $S\ge12$, $A\ge3$, $m=\lfloor S/3\rfloor$,
$h=\lceil\log_{A-1}m\rceil$, and
$\overline D\ge\max\{32,8(h+1)\}$. Let
$0<\epsilon\le1/32$ and $0<\delta\le1/16$.
Consider the finite family constructed below, whose kernels share one known
reward table and have at most $S$ states, $A$ actions per state, diameter at
most $\overline D$, and at most two positive successors per row. Suppose a
learner is correct on every member of this family: it returns a deterministic
stationary policy that is $\epsilon$-optimal from every state, jointly with
stopping by one common deterministic budget $T$, with probability at least
$1-\delta$ on each member.
Then
\[
 T\ge\frac{[m(A-1)-2]\overline D}{8192\epsilon^2}
           \log(1/\delta).
\]
The conclusion holds even with generative access and therefore also for
one continuing trajectory. Only correctness on this sparse hard family is
required, so a learner promised two-successor kernels is also covered.
Logarithms in this section are natural.
\end{theorem}
\begin{proof}
Put $b=A-1$ and $q=8/\overline D\le1/4$. Construct a rooted tree with
$m$ leaves, at most $m-1$ internal vertices, height at most $h$, and
at most $b$ children per vertex. One construction recursively divides
the leaves into at most $b$ nonempty groups of sizes differing by at most
one; every non-leaf has at least two children. Internal actions traverse
one parent or child edge deterministically, with unused actions holding.
The root has no parent. All these rewards are zero.

Replace leaf $j$ by a state $x_j$ and add a state $y_j$. At $x_j$, each
of $b$ candidate actions has reward one and moves to $y_j$ with probability
$p_{jl}$, otherwise holding. Its remaining action moves to its parent
with reward zero. At $y_j$ every action has reward zero and returns to
$x_j$ with probability $q$, otherwise holding. There are at most $3m-1$
states. In the base model all $p_{jl}=q$. Alternative $(j,l)$ changes
only $p_{jl}$ to $q(1-8\epsilon)$. All models have identical support
and known rewards.

A controlled route between any two states uses at most one initial
$y$-to-$x$ wait, at most $2h$ tree transitions, and at most one final
$x$-to-$y$ wait. Consequently its mean is at most
\[
 1/q+2h+1/[q(1-8\epsilon)]
 \le7\overline D/24+2h\le\overline D.
\]
The models communicate. Their diameters are also at least
$1/q=\overline D/8$.

Under a deterministic stationary policy, choosing a candidate at $x_j$
makes $\{x_j,y_j\}$ a closed irreducible pair with gain
$q/(q+p_{jl})$. Exiting at $x_j$ gives reward zero there. Every other
recurrent class therefore has gain at most $1/2$. In alternative $(j,l)$,
the changed pair has gain
\[
 \frac1{2-8\epsilon}>\frac12+2\epsilon.
\]
A policy routing every other state to this pair attains that gain, so it
is the all-state optimum. Conversely, all-state $\epsilon$-optimality
requires the output policy to have this pair as its sole recurrent class
and to choose action $l$ at $x_j$.

Call this global output event $E_{jl}$. It is determined by the output
policy and the common support, so it has the same definition under every
model. The $N=m(A-1)$ such events are disjoint. In particular, under
the base transcript law $\mathbb P_0$, at least $N-2$ of them have
$\mathbb P_0(E_{jl})\le1/2$. 

First consider a learner using at most $T$ samples. Include its independent
random tape in the transcript and pad after stopping with deterministic
null observations. If $N_{jl}(T)$ is the number of queries to the changed
row, conditional KL chain rules give
\[
 \mathrm{KL}(\mathbb P_0^T\Vert\mathbb P_{jl}^T)
 =\mathbb E_0N_{jl}(T)\,\mathrm{kl}(q,q(1-8\epsilon)).
\]
The action-selection rule contributes no divergence, even if queries are
adaptive. Bernoulli KL is bounded by its chi-squared divergence, giving
\[
 \mathrm{kl}(q,q(1-8\epsilon))
 \le\frac{(8q\epsilon)^2}{q(1-8\epsilon)[1-q(1-8\epsilon)]}
 \le256q\epsilon^2=2048\epsilon^2/\overline D.
\]
For an event with base probability $p\le1/2$ and alternative probability
$v\ge1-\delta$, binary data processing and binary entropy at most
$\log2$ imply
\[
\begin{aligned}
 \mathrm{KL}(\mathbb P_0^T\Vert\mathbb P_{jl}^T)
 &\ge\mathrm{kl}(p,v)\\
 &\ge\tfrac12\log(1/\delta)-\log2\\
 &\ge\tfrac14\log(1/\delta).
\end{aligned}
\]
Correctness in alternative $(j,l)$ supplies $v\ge1-\delta$. Sum the
resulting count lower bounds over the at least $N-2$ events. Since
$\sum_{jl}N_{jl}(T)\le T$, this gives the displayed claim.

Finally truncate the stated learner at its common budget $T$ and output
any deterministic policy on timeout. The joint guarantee implies that
this truncated learner remains correct with probability at least
$1-\delta$ on every model in the finite hard family. The fixed-budget
argument applies. No expected-time guarantee on rare failure histories
is assumed or inferred.
\end{proof}

Combining this lower bound with the diameter bound~\eqref{eq:diameter-bound}
gives the matching regimes in Appendix~\ref{app:matching-summary}.
For unrestricted support the sufficient condition is
$\epsilon^2\min\{S,\overline D\}\le1$; for at most two successors per
row it holds throughout the stated accuracy interval.

\paragraph{Matching consequence.}
The diameter bound~\eqref{eq:diameter-bound} matches up to
logarithms under the condition
\[
 \epsilon^2\min\{J_0/(SA),\overline D\}\le1.
\]
Its startup is at most
$\overline D\min\{J_0,SA\overline D\}$, so this condition absorbs it
into the precision term. The lower family has at most two successors per
row, so the two-successor class matches throughout the stated epsilon
range under the same remaining restrictions. For dense classes the
sufficient condition is $\epsilon^2\min\{S,\overline D\}\le1$.
This assertion fixes a reward across the hard kernels. Constant rewards
illustrate why the latter assertion would fail.

\section{Computational Complexity}\label{app:computation}
\subsection{One statistical batch}
\begin{lemma}[Bit cost of one statistical batch]\label{lem:stat-bit-cost}
For rational inputs as above, the rounded contraction fallback
computes the corrected certificate using
\[
M\,\operatorname{poly}(S,A,L_{\rm in},L_w,\log(M+2),
\log(1/\eta),\log(2/\epsilon))
\]
bit operations, where $L_w$ is the
encoding length of the frozen weights. This bound concerns one completed
batch and makes no assumption about the pilot or collector algorithm.
\end{lemma}
\begin{proof}
Put $\nu=\epsilon/128$. At discount gap $a$, choose a dyadic mesh
$d\le a^2\nu/128$. Sorting the $S$ successor values identifies each
quantile; rational sums compute its mean and variance; integer square-root
bracketing encloses the row functional to width at most $d$. Histograms
suffice, so a row evaluation does not expand the individual observations.
Round each Bellman update downward on this mesh. Row evaluation and
rounding cause at most $2d$ error per update. Starting at zero, sup-norm
error after $k$ updates is at most $\gamma^k/a+2d/a$.
For
\[
 k=\left\lceil\frac{\lceil\log_2(32/(a^2\nu))\rceil}{a}\right\rceil
\]
this is at most $3a\nu/64$. A further outward evaluation therefore
bounds the residual by $(1+\gamma)3a\nu/64+d<a\nu$,
which implies the required $R/a\le\nu$.
Both signs preserve the cube $[0,1/a]^S$. Fixed-mesh storage prevents
repeated growth of rational denominators across iterations; intermediate
sums have polynomial bit length in the displayed inputs and count sizes.
The smallest gap is at least $1/[2(M+1)]$ and there are
$O(\log(M+2))$ discounts. Thus at most
\(O(M\log(M+2)\log((M+2)/\epsilon))\)
operator evaluations suffice.
The virtual source-absorption grid contributes only to the logarithmic
confidence coefficient.
Together with $N\le3M/2$, these facts give the asserted bit bound.
\end{proof}

\subsection{Representation and full schedule}
\paragraph{Bit computation.}
At dyadic witnesses, coordinate bisection endpoints and unseen-set caps
have dyadic denominators of $O(\log(N+2))$ bits. The unseen sets form a chain; after contracting
lower bounds, sorted greedy filling under the laminar capacities computes
row support exactly in $O(S^2\log(N+2))$ rational operations with the
simple implementation. Rounded integer backups and scalar predictions
have polynomial bit length. On the good event the number of backups is
$N+O(\log(D+2)\log(D(N+1)+2))$.
At every completed statistical attempt $M\le H_j$ and
$N_{\rm stat}\le M+K\le3M/2$. The rounded contraction calculation in
Lemma~\ref{lem:stat-bit-cost} bounds its work by $H_j$
times a fixed polynomial in dimensions, input lengths and logarithms.
An incomplete collection performs no statistical solve. Summing across
the paid schedule proves the computational claim of the theorem.

\section{Comparison with Prior Work}\label{app:comparison}

\paragraph{Navigation constraints.}
\citet{navigation}, Sections 3--5, restrict information-directed allocations
to stationary flows and track a learned allocation along a controlled
trajectory. Their Theorem 6 concerns identifying a unique discounted-optimal
policy and bounds expected stopping time asymptotically as $\delta\to0$.
A finite uniform budget in diameter, support and accuracy requires
control of allocation learning and collection. We bound arbitrary
adaptive trajectory flows, handle decreasing auxiliary rewards with a
finite-cap controller, and convert nonuniform statistical requirements
into actual collection cost. These steps give Theorem~\ref{thm:main-total}'s
finite-time average-reward guarantee without structural inputs.
The principle of respecting navigation constraints in allocation is shared;
the new guarantee and its finite-time analysis are the distinction.

\paragraph{Variance-dependent online regret.}
FOCUS~\citep{focus26} is a tractable online comparator: its published
Theorem~3 bounds discounted regret by
$\widetilde O(\sqrt{SA\,\mathrm{Var}^*_\gamma}+\Gamma HSA)$
when the supplied $H$ bounds the true discounted optimal span; here
$\Gamma$ is the maximum row support size. Its average-reward and prior-free
consequences are important precedents.

\paragraph{Planner--collector composition.}
Theorem 4 of \citet{tdk} gives Online-DFE the interaction bound
$\widetilde O(SAD^2/\epsilon^2+S^2AD^3)$.
It is insufficient to compare only with this bound. Let
$H_* = \spn(h^*)$ for an optimal Bellman bias.
\citet{zc}, Algorithm 2 and Theorem 2, provide observable
fixed-accuracy stopping with generative access and
$\widetilde O((H_*+1)/\epsilon^2)$ samples per row, without knowing $H_*$.
GOSPRL~\cite{gosprl}, Theorem 1, collects $B$ prescribed samples online
at cost $\widetilde O(BD+S^2AD^{3/2})$, without knowing $D$.
Combining these guarantees gives the following comparator, whose
derivation we provide below:
\begin{equation}\label{eq:comparator}
 \widetilde O\bigl(SAD(H_*+1)/\epsilon^2+S^2AD^{3/2}\bigr).
\end{equation}

To see this, implement the generative algorithm's geometric row requests
$n_j=2^j$ with fresh GOSPRL calls, starting each from the current state.
Keep only the first requested successors of each row and charge every
intervening transition. Allocate $\delta/2$ to the simultaneous statistical
certificates and $\delta/[4(j+1)^2]$ to collection call $j$. Conditional
on each finite call-start history, unused row stacks have the product
law required by the statistical argument. Define their fixed-size
prefixes even when collection has not finished. The source's marginal
confidence unions therefore remain valid; neither cross-stage independence
nor conditioning on successful collection is needed.
Its proof supplies a deterministic sufficient row count
$n_* = \widetilde O((H_*+1)/\epsilon^2)$, used only in analysis.
Through the first geometric request exceeding $n_*$, the total requested
samples are $O(SAn_*)$ and the number of calls is logarithmic.
Sum the conditional collector bounds and intersect their events with
the statistical event. This yields~\eqref{eq:comparator}, retaining
the observable stop and all-state policy guarantee without a reset.

Since $H_*\le D$ and $D\ge1$, the diameter-only consequence of
\eqref{eq:comparator} is
$\widetilde O(SAD^2/\epsilon^2+S^2AD^{3/2})$.
The diameter-based upper bound improves its polynomial dependence on $D$ in both terms.
The comparison is between the displayed upper guarantees. For Theorem~\ref{thm:main-total},
$A_{\rm mix}\le S^2AD$ and, since $\epsilon\le1$ and $H_*+1\ge1$,
$SAD/\epsilon\le SAD(H_*+1)/\epsilon^2$.
Its displayed polynomial terms are therefore bounded, up to constants,
by those of~\eqref{eq:comparator}. This comparison of upper expressions
leaves instancewise runtime comparisons and other oracle inequalities open.

\paragraph{Other access and output criteria.}
SAVIC+ of \citet{lee}, Theorem 3.4, also gives
prior-free all-state policy identification on a single trajectory.
Writing $B_Q=\spn(Q^*)$, its guarantee is
$\widetilde O(t_{\rm cov}B_Q/\epsilon+
B_Q^2/(d_{\min}^{\pi_b}\epsilon^2))$ for $\epsilon\le1\le B_Q$.
Its cover time and minimum stationary state--action mass refer to the
behavior policy, so this is not a diameter-only comparison. Favorable
behavior distributions can give advantages not captured by our bound.

\citet{lrp}, Theorem 6, analyze expected
stopping for exact discounted policy identification under additional
assumptions, including a known reference transition measure.
\citet{jiao}, Section 3.1, use synchronous generative sampling.
\citet{fqi}, Section 4.3, analyze offline behavior-policy data
under stationary mixing assumptions. These results address adjacent questions.

\paragraph{Attribution and scope.}
The final contribution is an exploration history with simultaneous
coverage and affordable collection, combined with nonuniform statistical
stopping on the continuing trajectory. Unfinished-quota rewards and optimistic
sample collection have precedents in GOSPRL and its 0/1-UCRL comparison;
quantile monotonicity is explicitly imported from \citet{zzc} in
Appendix~\ref{app:nonuniform-statistic}.
Theorem~\ref{thm:lower} reconstructs the familiar diameter-based lower
order associated with~\cite{wwy}, using a bounded-time global output
event to match the stopping-and-output requirement used here. Full-range optimality and practical acceleration remain open.

\section{Experimental Setup}\label{app:experimental-details}
\subsection{Cyclic MDP}\label{app:cycle-protocol}
States are
$(i,b)$ with $i\in\{0,\ldots,m-1\}$ and $b\in\{0,1\}$, so $S=2m$,
and actions $0,1,2$. Actions $0,1$ receive known reward $b$; action $2$
receives zero. For action $a\in\{0,1\}$, the cycle index advances to
$i+1\pmod m$ with probability $q$ and otherwise stays at $i$; independently,
the next bit is one with probability $3/4$ if $a=i\bmod2$, and $1/4$
otherwise. Action $2$ always advances and sets the next bit by a fair coin.
Thus $q$ controls cycle progress under the rewarding actions. The Bellman equations give $g^*=3/4$ and $h^*(i,b)=b$, with span one.
We vary $S=16,32$, $q=1/16,1/4096$, and samples per row
$n=2^{12},2^{14},2^{16}$; nominal mass is $M=3Sn$.
Uniform random actions collect
every prescribed prefix along one continuing path. True-model quantities
appear only in simulation and post-hoc evaluation; the planner receives
counts and known rewards. Three fresh paired seeds per cell are fixed in
advance. Figure~\ref{fig:prefix-error} shows optimal-gain-minus-lower-certificate
error and its three-seed ranges.
The random seeds are 97001--97003. Numerical acceptance requires
verification of the single lower solve: exact outward residual distance $R/(1-\gamma)\le10^{-4}$ and
maximum row-interval width at most $10^{-4}$. This yields a numerically
verified lower certificate on the statistical event; it is not an
$\epsilon$-width upper/lower policy stopping test. The fixed-prefix rule uses
$\eta=1/20$, $\gamma=1-1/(M+1)$ with a numerical clamp inactive at the invariant fixed point; no diameter
is supplied to the solver.

\subsection{Connection to the bound}\label{app:experiment-bound}
For uniform row proportions, $w_i=1/(SA)$ and $n_i=n$, so the auxiliary
reward $1/w_i=SA$ is constant. The witness $z_0=0$ gives $G_P(w)=SA$
and $C_0=0$, hence $\kappa=\alpha/n$ and $c_z=0$ in
Theorem~\ref{thm:auxiliary-oracle}. The confidence factor depends only on
$S,A,n$ and the fixed confidence target.
Thus the theorem bounds $g_P^*-L_n$ by a multiple of
$\sqrt{(B_*+1)\alpha/n}$ at exact fixed points; the numerical residual
corrections are included in the reported $L_n$.
The experiment isolates this dependence at fixed sample sizes;
the theorem's constants and collection complexity require separate analysis. Under the optimal
rewarding policy $\pi^*$, the hitting radius is
\[
 \min_y\max_x\mathbb E_x^{\pi^*}\tau_y=(m-1)/q+m/3
\]
for the values of $q$ used here. A target with bit one attains the
minimum: reaching its cycle index takes at most $(m-1)/q$ steps in
expectation, with an additional $m/3$ from the chance of arriving with
bit zero. The resulting radius ratios are about 250 for both state sizes.

\clearpage
\section{Numerical Illustration}\label{app:additional-results}
Under the uniform-sampling protocol of Appendix~\ref{app:cycle-protocol},
$G_P(w)=SA$ and $C_0=0$, giving lower-certificate error
$O(\sqrt{(B_*+1)\alpha/n})$ in Theorem~\ref{thm:auxiliary-oracle}.
The varied quantity is the optimal policy's hitting radius, not controlled diameter.

All 36 solves satisfy the residual and row-interval criteria in
Appendix~\ref{app:cycle-protocol}. Figure~\ref{fig:prefix-error} includes
every seed and all 18 ratios paired by $S,n$ and seed, ranging from
$0.9889$ to $1.0232$. This is consistent with the bound on these instances.

\begin{figure}[H]\centering
\includegraphics[width=\linewidth]{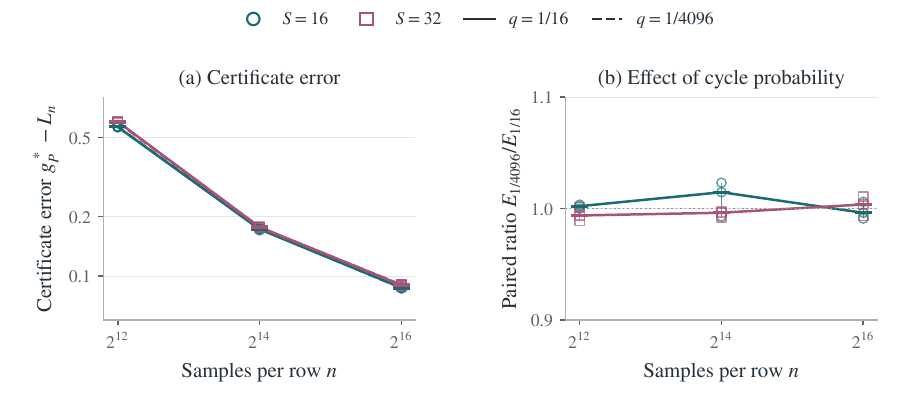}
\caption{Certificate errors are similar across cycle probabilities at fixed
span. (a) Absolute error $E_q=g_P^*-L_n$. (b) $E_{1/4096}/E_{1/16}$,
paired by $S,n$ and seed; one denotes equality. All three seeds, medians
and observed min--max are shown.
Color and shape identify $S$; line style identifies $q$. Policy gaps are reported below.}
\label{fig:prefix-error}\end{figure}

Six outputs at $q=1/4096$, $n=2^{16}$ have small nonzero policy gaps
(approximately $3.66\times10^{-5}$ to $1.83\times10^{-4}$); the comparison policy
used to define the hitting radius is the true optimal policy.

\end{document}